\documentclass[sigconf]{acmart}

\usepackage{booktabs}
\usepackage{tabularx}
\usepackage{array}
\usepackage{multirow}
\usepackage{graphicx}
\usepackage{xcolor}
\usepackage{colortbl}
\usepackage{algorithm}
\usepackage{algpseudocode}
\usepackage{multirow}

\AtBeginDocument{%
  }

\theoremstyle{plain}
\newtheorem{theorem}{Theorem}[section]
\newtheorem{proposition}[theorem]{Proposition}

\newtheorem{corollary}[theorem]{Corollary}
\theoremstyle{definition}

\theoremstyle{remark}

\renewcommand\footnotetextcopyrightpermission[1]{}

\begin{document}


\title{Split Federated Unlearning}

\author{%
\textbf{
    Jia Chen\textsuperscript{1,2,*},
    Yecheng Lv\textsuperscript{1,*},
    Xinrong Hu\textsuperscript{1},
    Mingdi Huang\textsuperscript{1},
    Jun Bai\textsuperscript{3,$\dagger$}
    }
}

\affiliation{%
  \institution{%
  \textsuperscript{1} School of Computer Science and Artificial Intelligence,
  Wuhan Textile University, Wuhan, China\\
  \textsuperscript{2} Engineering Research Center of Hubei Province
  for Clothing Information, Wuhan Textile University,
  Wuhan, China\\
    \textsuperscript{3} Australian Institute for Machine Learning (AIML),
    Adelaide University, Adelaide, Australia \\
  }
  \country{\unskip}
}



\begin{abstract}

Split federated learning (SFL) combines partitioned model execution with federated aggregation to support collaborative learning across distributed Web and edge environments. However, removing learned contributions in SFL is challenging due to cross-partition coupling and historical influence propagation through dual-side aggregation. To the best of our knowledge, we provide the first systematic formulation of \emph{Split Federated Unlearning} (SFU), encompassing client-wise and class-wise removal, counterfactual retraining, and partition-specific update constraints. We theoretically characterize historical influence propagation and mode-dependent unlearning feasibility, revealing structural limitations of restricted model updates. As an initial solution, we propose \emph{SURE}, integrating dual-importance estimation, selective parameter perturbation, and importance-guided recovery. Extensive experiments with existing unlearning methods adapted to SFL reveal substantial forgetting--utility trade-offs across model architectures, data distributions, and update modes. SURE achieves competitive agreement with mode-matched retraining at modest additional cost, while further analyses highlight the impact of split position and update constraints on unlearning behavior. These findings establish SFU as a distinct research problem and provide a foundation for unlearning in partitioned federated systems.

\end{abstract}

%



\keywords{
split federated unlearning,
split federated learning,
federated unlearning,
model partitioning,
importance-guided unlearning
}


\maketitle

\pagestyle{plain}
\thispagestyle{plain}
\begingroup
\renewcommand{\thefootnote}{\fnsymbol{footnote}}
\setcounter{footnote}{1}
\footnotetext{Jia Chen and Yecheng Lv contributed equally to this work.}

\setcounter{footnote}{2}
\footnotetext{Corresponding author: Jun Bai
(\texttt{jun.bai@adelaide.edu.au}).}
\endgroup

\section{Introduction}
\label{sec:introduction}

Personalized Web services increasingly rely on user interactions distributed across devices and data owners~\cite{qu2024usergoverned}. Federated learning (FL) enables collaborative training without centralizing raw data~\cite{mcmahan2017communication,bai2025unified}, but training complete models on resource-constrained clients can be costly~\cite{gao2022evaluation}. Split federated learning (SFL) addresses this limitation by partitioning model execution between clients and a server while retaining federated aggregation~\cite{thapa2022splitfed}. As illustrated in Fig.~\ref{fig:sfu_overview}(a), client-side models and their corresponding server-side replicas are jointly optimized, with each partition aggregated separately. This architecture offers a practical approach to collaborative learning in resource-heterogeneous Web and edge environments.


While SFL keeps user data local, it does not address the removal of their learned contributions when users withdraw data or leave a Web service. Deleting local records or ending participation cannot undo their accumulated influence on the trained model. Machine unlearning seeks to remove such influence without full retraining~\cite{bourtoule2021machine,romandini2025survey}. Existing federated unlearning methods, including FedEraser~\cite{liu2021federaser} and FedRecovery~\cite{zhang2023fedrecovery}, address contributions propagated through federated aggregation, while split unlearning considers removal under partitioned model execution~\cite{jiang2025split,etheridge2025sequential}. Neither formulation directly captures SFL, where both mechanisms operate together.

Combining split execution with dual-side aggregation introduces distinct challenges for unlearning in SFL. Fig.~\ref{fig:sfu_overview}(b) illustrates client-wise and class-wise removal, with full retained-data retraining as the ideal counterfactual reference and three model update modes. As highlighted in Fig.~\ref{fig:sfu_overview}(c), three challenges arise: \emph{cross-partition coupling}, as both partitions interact through exchanged activations and gradients; \emph{historical influence propagation}, as repeated dual-side aggregation propagates contributions across clients and rounds; and \emph{mode-dependent feasibility}, as restricting updates to one partition may limit attainable forgetting. These intertwined challenges motivate a dedicated formulation of \emph{Split Federated Unlearning} (SFU).

\begin{figure*}[t]
    \centering
    \includegraphics[width=\textwidth]{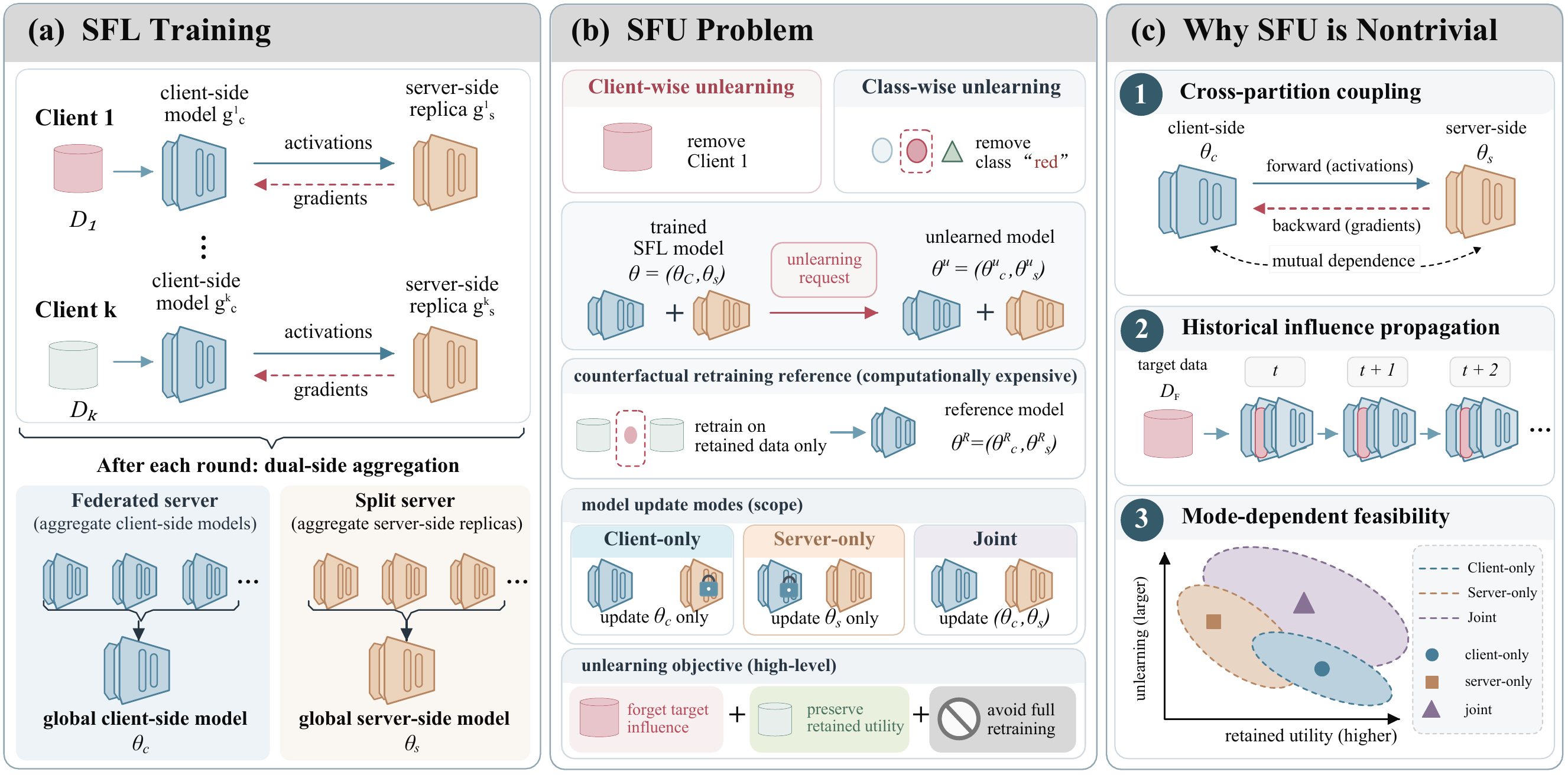}
    \setlength{\abovecaptionskip}{-5pt}
    \caption{Overview of Split Federated Unlearning (SFU).
    (a) SFL jointly trains client-side models and server-side replicas, with independent aggregation of both partitions.
    (b) SFU considers client-wise and class-wise removal, using full retained-data retraining as the counterfactual reference under different model update modes.
    (c) SFU involves cross-partition coupling, historical influence propagation, and mode-dependent unlearning feasibility.}
    \label{fig:sfu_overview}
\end{figure*}

To address these challenges, we formalize SFU for client-wise and class-wise removal, distinguishing forgetting requests from the permitted model updates (Client-only, Server-only, and Joint). Full retained-data retraining serves as the ideal counterfactual reference, while mode-matched retraining captures the effects of restricted updates. We further characterize how forgotten-data influence propagates through coupled dual-side training and analyze mode-dependent unlearning feasibility through local prediction-discrepancy bounds and excess retained-data risk under a quadratic approximation. These results identify structural limitations imposed by model partitioning, independently of any particular unlearning algorithm.


To investigate SFU in practice, we propose \emph{SURE}, an initial SFU approach combining dual-importance estimation, selective parameter perturbation, and retained-data recovery. SURE identifies target-relevant parameters while accounting for retained knowledge, applies targeted perturbations within the permitted model partitions, and performs importance-guided recovery through the original SFL training protocol. We systematically evaluate SURE and existing unlearning methods adapted to SFL across different datasets, data distributions, and model update modes. Experiments on MNIST, CIFAR-10, and CIFAR-100 reveal substantial forgetting--utility trade-offs among adapted baselines, while SURE achieves competitive agreement with mode-matched retraining at modest additional cost. 
Further analyses examine the effects of split position, component design, and multiple client removals.




Our main contributions are threefold:
(1) To the best of our knowledge, we provide the first systematic formulation of \emph{Split Federated Unlearning} (SFU), jointly accounting for cross-partition coupling, dual-side aggregation, and partition-specific update constraints under client-wise and class-wise removal.
(2) We theoretically characterize historical influence propagation and mode-dependent unlearning feasibility, revealing how accumulated contributions and restricted updates affect attainable forgetting behavior.
(3) We propose \emph{SURE} as an initial importance-guided solution and systematically evaluate adapted unlearning methods in SFL, examining forgetting--utility trade-offs, model partitioning, and system efficiency.

\section{Related Work}



\paragraph{Federated Unlearning.}
Machine unlearning aims to remove learned data contributions without full retraining~\cite{bourtoule2021machine,li2026cross}. In FL~\cite{mcmahan2017communication,bai2025non}, unlearning is complicated by historical contributions accumulated through model aggregation. Early approaches explored update calibration and selective retraining~\cite{liu2021federaser,liu2022right}, alongside privacy-preserving model recovery~\cite{zhang2023fedrecovery}. Subsequent work investigated class-specific pruning, incentive-aware selection, subgraph removal, and privacy-aware recovery~\cite{wang2022pruning,lin2024incentive,liu2025subgraph,wang2026diary}. Recent methods further address gradient conflicts, server-side unlearning, and task-vector alignment~\cite{pan2025fedosd,li2026fedcare,nguyen2026foul,weng2026gdfa}. However, these approaches do not explicitly account for dependencies between separately aggregated model partitions.



\paragraph{Split Learning and Unlearning.}
Split learning (SL) partitions model execution between clients and servers through intermediate activations and gradients~\cite{vepakomma2018split}. Recent split unlearning methods address training dependencies through isolation-based removal~\cite{jiang2025split} and gradient-cache projection with retained-data fine-tuning~\cite{etheridge2025sequential}. Parameter-selective approaches, such as SalUn~\cite{fan2024salun}, further demonstrate the role of weight saliency in targeted unlearning. Nevertheless, these methods do not jointly characterize split-model dependencies and repeated federated aggregation.



\paragraph{Split Federated Learning.}
SFL combines partitioned execution with federated aggregation to support resource-constrained clients~\cite{thapa2022splitfed}. Subsequent studies have improved scalability through server-side parallelism~\cite{he2025hourglass} and investigated training instability under heterogeneous data~\cite{tirana2026heterogeneity}. However, such studies primarily address training efficiency and knowledge preservation rather than deliberate removal of historical contributions. Unlearning in SFL must additionally account for cross-partition dependencies, dual-side aggregation, and partition-specific update constraints.

\section{Split Federated Unlearning}
\label{sec:problem}

\subsection{System Model}
\label{sec:system_model}

We consider an SFL system~\cite{thapa2022splitfed} with $K$ clients, a split server, and a federated aggregation server. Client $k\in\mathcal{K}=\{1,\ldots,K\}$ holds data $\mathcal{D}_k$ and a client-side model $f_C(\cdot;\theta_{C,k})$; the split server maintains a corresponding replica $f_S(\cdot;\theta_{S,k})$. The predictor is $f_k(x)=f_S(f_C(x;\theta_{C,k});\theta_{S,k})$. Training exchanges cut-layer activations and gradients, followed by independent aggregation of both partitions each round:
\begin{equation}
\theta_s^{t+1}=\sum_{k\in\mathcal{K}_t}w_{k,t}^{s}\theta_{s,k}^{t+1},
\qquad s\in\{C,S\},
\label{eq:sfl_aggregation}
\end{equation}
where $\mathcal{K}_t$ is the participating client set and the weights $w_{k,t}^{s}$ are normalized. Client-side models are aggregated by the federated aggregation server and server-side replicas by the split server; both aggregated partitions initialize the next round. Thus, split computation couples the two partitions, while aggregation propagates client contributions across rounds.

\subsection{Problem Formulation}
\label{sec:problem_formulation}

\paragraph{Unlearning Requests.}
Let $\mathcal{D}=\bigcup_{k\in\mathcal{K}}\mathcal{D}_k$ and $\mathcal{D}_R=\mathcal{D}\setminus\mathcal{D}_F$. We consider client-wise removal of client $f$ and class-wise removal of class $c_f$:
\begin{equation}
\mathcal{D}_F=\begin{cases}
\mathcal{D}_f, & \text{client-wise},\\
\displaystyle\bigcup_{k\in\mathcal{K}}
\{(x,y)\in\mathcal{D}_k:y=c_f\}, & \text{class-wise}.
\end{cases}
\label{eq:unlearning_requests}
\end{equation}
The withdrawing client is excluded from subsequent participation; for class-wise removal, clients with retained data may remain.

\paragraph{Counterfactual Reference.}
Let $\mathcal{A}_{\mathrm{SFL}}$ denote the original training procedure, including split execution and dual-side aggregation. Define
\begin{equation}
\theta^*=\mathcal{A}_{\mathrm{SFL}}(\mathcal{D};\xi),\qquad
\theta^R=\mathcal{A}_{\mathrm{SFL}}(\mathcal{D}_R;\xi'),
\label{eq:retraining_reference}
\end{equation}
where $\theta=(\theta_C,\theta_S)$ and $\xi,\xi'$ represent training randomness. The ideal reference $\theta^R$ is trained from initialization on retained data using the same SFL procedure, with participation and aggregation weights updated accordingly. Given $\theta^*$, $\mathcal{D}_R$, and admissible request information $\mathcal{I}_F$, an unlearning method outputs $\theta^U=\mathcal{U}(\theta^*,\mathcal{D}_R,\mathcal{I}_F)$. Exact unlearning requires agreement with the retraining distribution~\cite{bourtoule2021machine}; here we study approximate unlearning without full retraining.

\paragraph{Model Update Modes.}
We distinguish $\mathcal{M}\in\{\{C\},\{S\},\{C,S\}\}$, denoting \emph{Client-only}, \emph{Server-only}, and \emph{Joint} updates, independently of the forgetting request. Any frozen partition, including its client-specific replicas and persistent buffers $b_s$, remains at the original checkpoint:
\begin{equation}
(\theta_s^{(t)},b_s^{(t)})=(\theta_s^*,b_s^*),
\quad s\notin\mathcal{M},\ \forall t.
\label{eq:model_update_modes}
\end{equation}
Frozen partitions still participate in split computation but are not updated.

\paragraph{Unlearning Objectives.}
SFU seeks to approximate full retraining while preserving retained-data utility and limiting computation and communication. We measure predictive discrepancy by
\begin{equation}
\Delta_F(\theta^U,\theta^R)=
\mathbb{E}_{x\sim\mathcal{Q}}
\big[d(f(x;\theta^U),f(x;\theta^R))\big],
\label{eq:forgetting_discrepancy}
\end{equation}
where $\mathcal{Q}$ is an evaluation distribution and $d$ a prediction discrepancy. The remaining objectives are retained-data risk $\mathcal{L}_R(\theta^U)$ and unlearning cost $\mathcal{C}(\mathcal{U})$, subject to $\mathcal{M}$. For evaluation, \emph{mode-matched Retrain} reinitializes and retrains only the permitted partition on $\mathcal{D}_R$, freezing the other at $\theta^*$. It isolates update-mode restrictions but does not replace full retraining as the ideal reference; neither reference is assumed available to $\mathcal{U}$.

\subsection{Theoretical Characterization}
\label{sec:theoretical_analysis}

\paragraph{Historical Influence Propagation.}
Training-data influence has been studied through algorithmic stability and influence functions~\cite{hardt2016train,koh2017influence}. We examine how such influence propagates through successive rounds of coupled SFL training. Let $z^t=(\theta_C^t,\theta_S^t,\omega^t)$ denote the synchronized training state at round $t$, where $\omega^t$ includes auxiliary optimizer states and persistent buffers.
Under coupled randomness and a common initialization, the original and retained-only trajectories satisfy
\begin{equation}
z^{t+1}=\mathcal{G}_t(z^t),\qquad
z_R^{t+1}=\mathcal{G}_t^R(z_R^t),
\label{eq:sfu_dynamics}
\end{equation}
where each round mapping includes split updates and dual-side aggregation. Set $\delta_t=z^t-z_R^t$.

\begin{theorem}[Historical Influence Propagation]
\label{thm:influence_propagation}
If $\mathcal{G}_t$ is continuously differentiable along the segment from $z_R^t$ to $z^t$, define
\begin{equation}
\begin{aligned}
A_t&=\int_0^1 D\mathcal{G}_t(z_R^t+\alpha\delta_t)\,d\alpha,\\
b_t&=\mathcal{G}_t(z_R^t)-\mathcal{G}_t^R(z_R^t).
\end{aligned}
\label{eq:sfu_influence_terms}
\end{equation}
Then, with $\delta_0=0$,
\begin{equation}
\delta_T=\sum_{t=0}^{T-1}
(A_{T-1}\cdots A_{t+1})b_t,
\label{eq:sfu_influence_decomposition}
\end{equation}
where empty products equal the identity. If $\|A_t\|_2\le L_t$, then
\begin{equation}
\|\delta_T\|_2\le
\sum_{t=0}^{T-1}\left(\prod_{r=t+1}^{T-1}L_r\right)\|b_t\|_2.
\label{eq:sfu_influence_bound}
\end{equation}
\end{theorem}

Theorem~\ref{thm:influence_propagation} decomposes historical influence into round-wise discrepancies $b_t$ and their propagation through subsequent updates (proof in Appendix~\ref{app:proof-influence}). Here, $b_t$ also captures changes in client participation and aggregation weights. For a single local gradient step, the cross-partition Jacobian blocks satisfy
\begin{equation}
\begin{aligned}
J_{CS}^{t}&=-\eta_C\!\sum_{k\in\mathcal{K}_t}
 w_{k,t}^{C}\nabla_{\theta_C\theta_S}^{2}\ell_{k,t},\\
J_{SC}^{t}&=-\eta_S\!\sum_{k\in\mathcal{K}_t}
 w_{k,t}^{S}\nabla_{\theta_S\theta_C}^{2}\ell_{k,t},
\end{aligned}
\label{eq:sfu_cross_jacobian}
\end{equation}
where $\ell_{k,t}$ is the local split-model loss. Nonzero mixed derivatives couple the partition dynamics despite separate aggregation. Consequently, removing only a client's final-round update does not generally reproduce the retained-only training trajectory.


\paragraph{Mode-Dependent Unlearning Feasibility.}
We characterize attainable prediction changes through local model linearization, a perspective also used to analyze neural-network training dynamics~\cite{lee2019wide}. Let $F(\theta)=f(\cdot;\theta)$ belong to $\mathcal{H}=L^2(\mathcal{Q};\mathbb{R}^m)$, with retraining shift $r=F(\theta^R)-F(\theta^*)$.
For mode $\mathcal{M}$, let $\Theta_{\mathcal{M}}=\{\Delta\theta:\Delta\theta_s=0\ \forall s\notin\mathcal{M}\}$ and $\mathcal{V}_{\mathcal{M}}=\operatorname{range}(DF(\theta^*)|_{\Theta_{\mathcal{M}}})$. The best achievable discrepancy within radius $\rho$ is
\begin{equation}
\mathcal{E}_{\mathcal{M}}(\rho)=
\inf_{\substack{\Delta\theta\in\Theta_{\mathcal{M}}\\\|\Delta\theta\|_2\le\rho}}
\|F(\theta^*+\Delta\theta)-F(\theta^R)\|_{\mathcal{H}}.
\label{eq:sfu_feasibility_error}
\end{equation}

\begin{theorem}[Local Unlearning Feasibility Bound]
\label{thm:feasibility_bound}
Suppose $F$ is continuously differentiable near $\theta^*$ with an $L$-Lipschitz derivative on $B(\theta^*,\rho)$. Let $P_{\mathcal{M}}$ project orthogonally onto $\mathcal{V}_{\mathcal{M}}$ and $d_{\mathcal{M}}=\|(I-P_{\mathcal{M}})r\|_{\mathcal{H}}$. Then
\begin{equation}
\mathcal{E}_{\mathcal{M}}(\rho)
\ge\left[d_{\mathcal{M}}-\frac{L\rho^2}{2}\right]_+.
\label{eq:sfu_feasibility_bound}
\end{equation}
\end{theorem}


The residual $d_{\mathcal{M}}$ measures the retraining shift inaccessible to first-order updates within the permitted partitions. If $d_{\mathcal{M}}>L\rho^2/2$, exact predictive agreement is impossible within the update ball (proof in Appendix~\ref{app:proof-feasibility}).

\begin{corollary}[Unlearning Mode Hierarchy]
\label{cor:mode_hierarchy}
Joint updates include both single-partition feasible sets. Hence, for the same $\rho$,
\begin{equation}
\begin{aligned}
d_{\{C,S\}}&\le\min\{d_{\{C\}},d_{\{S\}}\},\\
\mathcal{E}_{\{C,S\}}(\rho)&\le
\min\{\mathcal{E}_{\{C\}}(\rho),\mathcal{E}_{\{S\}}(\rho)\}.
\end{aligned}
\label{eq:sfu_feasibility_hierarchy}
\end{equation}
These bounds concern attainable rather than algorithm-specific performance (proof in Appendix~\ref{app:proof-hierarchy}).
\end{corollary}

\paragraph{Excess Risk under Single-Partition Updates.}
To examine the effect of freezing one model partition, we consider a local positive-definite quadratic surrogate around the retrained solution, following standard second-order optimization analysis~\cite{nocedal2006numerical}, defined as:
\begin{equation}
Q_R(\theta)=\mathcal{L}_R(\theta^R)
+\tfrac12(\theta-\theta^R)^\top H_R(\theta-\theta^R),
\label{eq:sfu_quadratic_risk}
\end{equation}
where $H_R=\bigl[\begin{smallmatrix}H_{CC}&H_{CS}\\H_{SC}&H_{SS}\end{smallmatrix}\bigr]\succ0$.

\begin{proposition}[Single-Partition Excess-Risk Gap]
\label{prop:excess_risk}
Let $e_C=\theta_C^*-\theta_C^R$ and $e_S=\theta_S^*-\theta_S^R$. Under this surrogate, the minimum excess retained-data risks for Client-only and Server-only updates are, respectively,
\begin{align}
\mathcal{R}_{C}
&=\tfrac12 e_S^\top
(H_{SS}-H_{SC}H_{CC}^{-1}H_{CS})e_S,
\label{eq:sfu_client_risk_gap}\\
\mathcal{R}_{S}
&=\tfrac12 e_C^\top
(H_{CC}-H_{CS}H_{SS}^{-1}H_{SC})e_C.
\label{eq:sfu_server_risk_gap}
\end{align}
Both are nonnegative, and strictly positive for nonzero displacement in the respective frozen partition.
\end{proposition}



The Schur-complement terms~\cite{zhang2005schur} quantify the residual risk after the updatable partition optimally compensates for the frozen one. Thus, cross-partition adaptation can reduce, but not necessarily eliminate, the excess retained-data risk under the quadratic surrogate (proof in Appendix~\ref{app:proof-risk}).

\begin{figure*}[t]
    \centering
    \includegraphics[width=\textwidth]{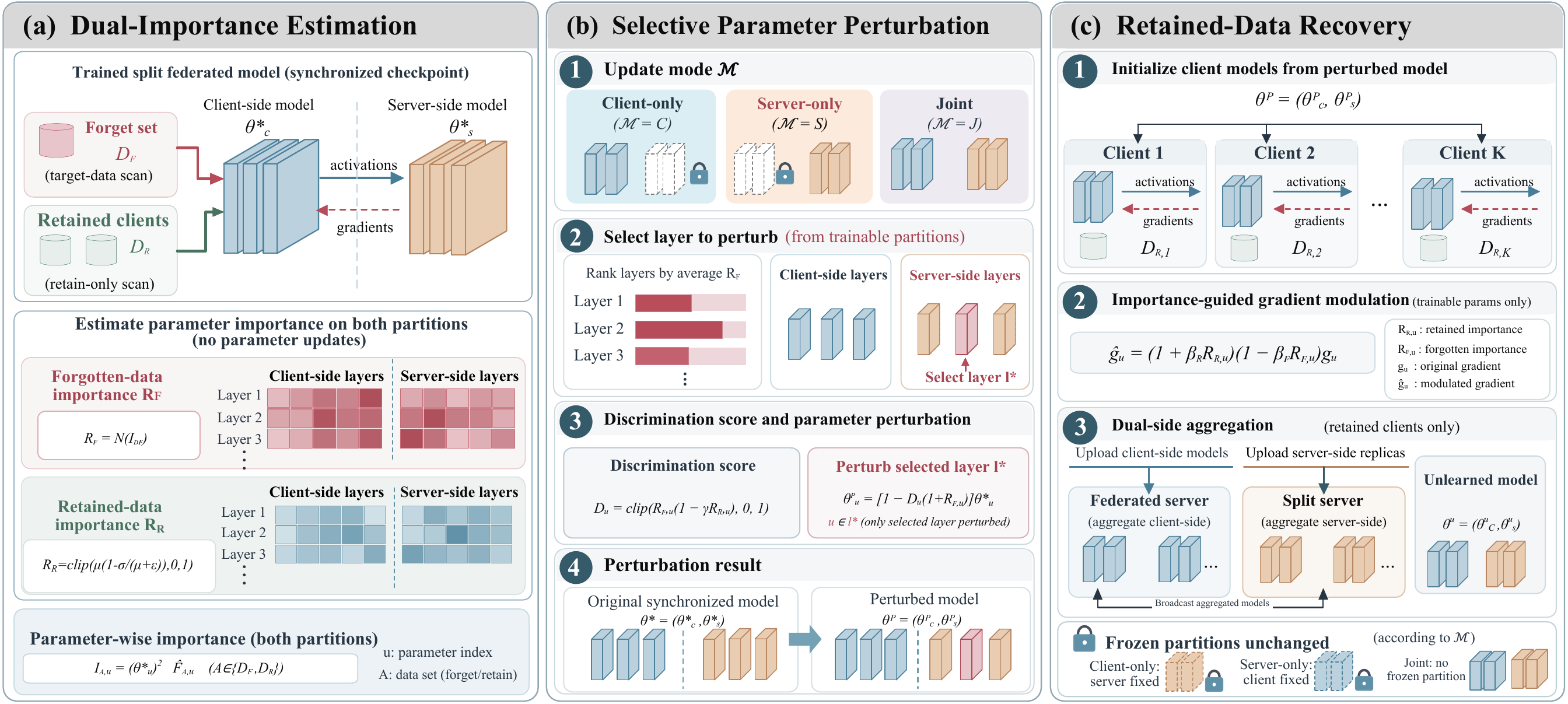}
    \setlength{\abovecaptionskip}{-5pt}
    
\caption{Overview of SURE.
(a) Dual-importance estimation identifies parameters relevant to forgotten and retained data across both model partitions.
(b) Mode-aware perturbation targets one layer within the permitted partitions.
(c) Retained-data recovery employs importance-guided gradient modulation and dual-side aggregation while preserving frozen partitions.}

    \label{fig:sure_framework}
\end{figure*}

\section{SURE: An Initial Solution}
\label{sec:method}

SURE provides an importance-guided approximation to SFU without reconstructing the original training trajectory. Starting from the synchronized checkpoint $\theta^*=(\theta_C^*,\theta_S^*)$, it estimates forgotten- and retained-data importance, selectively perturbs a target-relevant layer, and recovers retained utility through constrained split-federated training. 
The three stages are illustrated in Fig.~\ref{fig:sure_framework}(a)--(c), with the complete procedure summarized in Algorithm~\ref{alg:sure}. Throughout, only partitions permitted by update mode $\mathcal{M}$ may change.

\begin{algorithm}[t]
\caption{SURE for Split Federated Unlearning}
\label{alg:sure}
\begin{algorithmic}[1]
\Require $\theta^*$, $\mathcal{D}_F$, $\{\mathcal{D}_{k,R}\}$, $\mathcal{M}$, $T_U$, validation set $\mathcal{V}$, forgetting constraint $\mathcal{C}_F$
\Ensure Unlearned model $\theta^U$
\State Initialize replicas from $\theta^*$; compute $R_F$ via Eqs.~\eqref{eq:sure_fisher}--\eqref{eq:sure_forget_importance}
\State Compute $R_{k,R}$ for each retained client and $R_R$ via Eq.~\eqref{eq:sure_retained_importance}
\State Select $\ell^*$ using Eq.~\eqref{eq:sure_layer_selection}
\State Compute $D_u$ and $\theta^P$ using Eqs.~\eqref{eq:sure_discrimination}--\eqref{eq:sure_perturbation}
\State $\theta^0\gets\theta^P$
\For{$r=0,\ldots,T_U-1$}
    \State Initialize retained-client replicas from $\theta^r$
    \State Perform retained-only split training; modulate trainable gradients via Eq.~\eqref{eq:sure_gradient_modulation}
    \State Aggregate only partitions in $\mathcal{M}$ using Eq.~\eqref{eq:sure_recovery_aggregation}; preserve frozen states
\EndFor
\State Select $\theta^U\in\{\theta^r\}_{r=0}^{T_U}$ using $\mathcal{V}$ and $\mathcal{C}_F$ (Appendix~\ref{app:experimental-settings})
\State \Return $\theta^U$
\end{algorithmic}
\end{algorithm}

\subsection{Dual-Importance Estimation}
\label{sec:importance}

As shown in Fig.~\ref{fig:sure_framework}(a), SURE scans forgotten and retained data through the original split computation path, with all replicas initialized from $\theta^*$, without parameter updates or aggregation. For dataset $\mathcal{A}$ and parameter $u$ in either partition, we use a mini-batch gradient-squared proxy for diagonal Fisher information:
\begin{equation}
\widehat F_{\mathcal{A},u}=
\frac{1}{|\mathcal{A}|}\sum_{B\in\mathcal{B}(\mathcal{A})}
|B|\left(\frac{\partial\mathcal{L}_B(\theta^*)}{\partial\theta_u}\right)^2,
\label{eq:sure_fisher}
\end{equation}
where $\mathcal{B}(\mathcal{A})$ contains mini-batches and $\mathcal{L}_B$ is the average batch loss. To account for the subsequent multiplicative perturbation, we weight sensitivity by parameter magnitude:
\begin{equation}
I_{\mathcal{A},u}=(\theta_u^*)^2\widehat F_{\mathcal{A},u}.
\label{eq:sure_importance}
\end{equation}
The forgotten-data map is
\begin{equation}
R_{F,u}=\mathcal{N}(I_{\mathcal{D}_F,u}),
\label{eq:sure_forget_importance}
\end{equation}
where $\mathcal{N}$ is a fixed normalization to $[0,1]$. This applies to either client-wise or class-wise forget sets. For each retained client $k\in\mathcal{K}_R$, we likewise compute $R_{k,R,u}=\mathcal{N}(I_{\mathcal{D}_{k,R},u})$. With $n_R=|\mathcal{K}_R|$, define
\begin{equation}
\begin{aligned}
\mu_u&=n_R^{-1}\sum_{k\in\mathcal{K}_R}R_{k,R,u},\\
\sigma_u&=\sqrt{n_R^{-1}\sum_{k\in\mathcal{K}_R}
(R_{k,R,u}-\mu_u)^2}.
\end{aligned}
\label{eq:sure_retained_statistics}
\end{equation}
The common retained-data importance is
\begin{equation}
R_{R,u}=\operatorname{clip}\!\left(
\mu_u\left(1-\frac{\sigma_u}{\mu_u+\epsilon}\right),0,1\right),
\label{eq:sure_retained_importance}
\end{equation}
where $\epsilon>0$ ensures numerical stability. The variance discount favors parameters consistently important across retained clients.

\subsection{Selective Parameter Perturbation}
\label{sec:perturbation}

Fig.~\ref{fig:sure_framework}(b) illustrates mode-aware perturbation. Among layers $\mathcal{L}_{\mathcal{M}}$ in the updatable partitions, SURE selects a single layer with the highest mean forgotten-data importance:
\begin{equation}
\ell^*=\arg\max_{\ell\in\mathcal{L}_{\mathcal{M}}}
\frac{1}{|\ell|}\sum_{u\in\ell}R_{F,u}.
\label{eq:sure_layer_selection}
\end{equation}
Thus, Joint mode considers both partitions but still perturbs only one layer. Within $\ell^*$, the discrimination score is
\begin{equation}
D_u=\operatorname{clip}\big(R_{F,u}(1-\gamma R_{R,u}),0,1\big),
\label{eq:sure_discrimination}
\end{equation}
where $\gamma\ge0$ discounts parameters important to retained data; clipping also permits $\gamma>1$. SURE then scales the selected parameters:
\begin{equation}
\theta_u^P=
\begin{cases}
[1-D_u(1+R_{F,u})]\theta_u^*, & u\in\ell^*,\\
\theta_u^*, & u\notin\ell^*.
\end{cases}
\label{eq:sure_perturbation}
\end{equation}
Larger target-specific scores induce stronger changes, potentially including sign reversal. Other layers and frozen partitions remain unchanged. The resulting $\theta^P$ initializes recovery; perturbation alone does not establish unlearning.

\subsection{Retained-Data Recovery}
\label{sec:recovery}

Recovery trains exclusively on $\mathcal{D}_R$ using the original SFL protocol (Fig.~\ref{fig:sure_framework}(c)). To restore retained utility while limiting reversal of the intervention, gradients $g_u$ are modulated as
\begin{equation}
\widetilde g_u=(1+\beta_R R_{R,u})(1-\beta_F R_{F,u})g_u,
\label{eq:sure_gradient_modulation}
\end{equation}
where $\beta_R\ge0$ and $\beta_F\in[0,1]$ respectively reinforce retained-relevant updates and attenuate forgotten-relevant ones. Optionally, importance-weighted stability penalties constrain drift from $\theta^*$:
\begin{equation}
\Omega_A(\theta,\theta^*)=\tfrac12\sum_{u\in\mathcal{P}_{\mathcal{M}}}
R_{A,u}(\theta_u-\theta_u^*)^2,\quad A\in\{F,R\},
\label{eq:sure_stability}
\end{equation}
where $\mathcal{P}_{\mathcal{M}}$ is the set of trainable parameters. The recovery loss is
\begin{equation}
\mathcal{L}_{\mathrm{rec}}=\mathcal{L}_R(\theta)
+\lambda_R\Omega_R(\theta,\theta^*)
+\lambda_F\Omega_F(\theta,\theta^*).
\label{eq:sure_recovery_objective}
\end{equation}
Each recovery round initializes client-specific replicas from the current synchronized state, performs split training on retained clients, and aggregates the two permitted partitions separately:
\begin{equation}
\theta_s^{r+1}=
\begin{cases}
\displaystyle\sum_{k\in\mathcal{K}_R^r}
w_{k,r}^{s,R}\theta_{s,k}^{r+1}, & s\in\mathcal{M},\\
\theta_s^*, & s\notin\mathcal{M},
\end{cases}
\label{eq:sure_recovery_aggregation}
\end{equation}
where $\mathcal{K}_R^r$ and $w_{k,r}^{s,R}$ denote participating retained clients and normalized aggregation weights. Frozen parameters, replicas, and persistent buffers remain fixed, although frozen partitions still support split computation. The final checkpoint $\theta^U$ is selected using validation data and the applicable forgetting constraint (Appendix~\ref{app:experimental-settings}). SURE targets approximate, not certified, unlearning.

\providecommand{\msd}[2]{\ensuremath{#1_{\scriptscriptstyle\pm #2}}}
\providecommand{\msdb}[2]{\ensuremath{\mathbf{#1}_{\scriptscriptstyle\pm #2}}}

\begin{table*}[t]
\centering
\caption{Client-wise unlearning under extreme non-IID partitions (mean $\pm$ SD, \%). AG (Avg. Gap)$\downarrow$ is the mean absolute distance of Ret., For., Test, and MIA from mode-matched Retrain (percentage points), computed from reported means. Ret./Test$\uparrow$: higher is better. Bold marks the best non-reference Ret./Test or AG within each dataset--model and update mode. 
}

\label{tab:clientwise-extreme}
\begingroup
\small
\setlength{\tabcolsep}{1pt}
\renewcommand{\arraystretch}{0.1}
\resizebox{\textwidth}{!}{%
\begin{tabular}{@{}l*{15}{c}@{}}
\toprule
& \multicolumn{5}{c}{Client-only} & \multicolumn{5}{c}{Server-only} & \multicolumn{5}{c}{Joint} \\
\cmidrule(lr){2-6}\cmidrule(lr){7-11}\cmidrule(l){12-16}
Method & Ret.$\uparrow$ & For. & Test$\uparrow$ & MIA & AG$\downarrow$ & Ret.$\uparrow$ & For. & Test$\uparrow$ & MIA & AG$\downarrow$ & Ret.$\uparrow$ & For. & Test$\uparrow$ & MIA & AG$\downarrow$ \\
\midrule
\multicolumn{16}{c}{\textit{MNIST / LeNet}} \\
\cmidrule(lr){1-16}
\textit{Retrain} & \msd{88.24}{1.48} & \msd{96.97}{0.35} & \msd{89.36}{0.37} & \msd{44.26}{0.68} & 0.00 & \msd{88.35}{2.24} & \msd{0.00}{0.00} & \msd{80.17}{0.47} & \msd{53.75}{0.92} & 0.00 & \msd{89.71}{2.41} & \msd{0.00}{0.00} & \msd{81.07}{0.74} & \msd{52.33}{0.98} & 0.00 \\
\cmidrule(lr){1-16}
Fine-tune & \msd{88.03}{1.26} & \msd{96.97}{4.96} & \msdb{89.62}{2.18} & \msd{46.70}{2.98} & 0.73 & \msd{92.04}{1.63} & \msd{4.74}{2.32} & \msd{83.42}{1.25} & \msd{52.93}{3.92} & 3.13 & \msdb{93.87}{1.52} & \msd{17.03}{14.96} & \msd{86.17}{1.62} & \msd{53.40}{1.98} & 6.84 \\
FedEraser & \msd{86.86}{2.58} & \msd{96.54}{3.65} & \msd{87.86}{2.75} & \msd{44.07}{3.27} & 0.88 & \msd{81.95}{2.72} & \msd{0.00}{0.00} & \msd{73.72}{1.35} & \msd{49.74}{3.62} & 4.22 & \msd{81.12}{1.21} & \msd{0.00}{0.00} & \msd{73.02}{1.65} & \msd{47.83}{2.25} & 5.29 \\
FedRecovery & \msd{86.00}{0.86} & \msd{96.54}{0.43} & \msd{87.14}{0.94} & \msd{50.59}{5.77} & 2.81 & \msd{86.06}{1.23} & \msd{95.67}{0.21} & \msdb{87.14}{1.62} & \msd{46.70}{5.26} & 28.00 & \msd{86.29}{1.23} & \msd{95.24}{0.43} & \msdb{86.60}{0.34} & \msd{48.18}{4.62} & 27.09 \\
SalUn & \msd{82.74}{0.63} & \msd{93.07}{1.63} & \msd{84.10}{0.71} & \msd{45.39}{1.73} & 3.95 & \msd{11.02}{1.63} & \msd{0.00}{0.00} & \msd{9.96}{3.82} & \msd{51.83}{1.41} & 37.37 & \msd{15.02}{2.36} & \msd{0.00}{0.00} & \msd{21.96}{3.26} & \msd{51.53}{1.24} & 33.65 \\
MoDe & \msd{87.26}{1.28} & \msd{97.40}{0.42} & \msd{88.63}{1.28} & \msd{52.52}{2.95} & 2.60 & \msd{87.08}{1.78} & \msd{0.00}{0.00} & \msd{78.42}{2.32} & \msd{51.48}{2.35} & 1.32 & \msd{85.63}{0.24} & \msd{0.00}{0.00} & \msd{76.94}{2.32} & \msd{52.23}{4.22} & 2.08 \\
FedOSD & \msd{85.97}{2.51} & \msd{95.67}{0.98} & \msd{87.05}{1.57} & \msd{51.51}{2.26} & 3.28 & \msd{86.76}{2.31} & \msd{0.00}{0.00} & \msd{78.89}{1.62} & \msd{51.58}{1.22} & \textbf{1.26} & \msd{88.71}{2.21} & \msd{0.00}{0.00} & \msd{80.61}{1.27} & \msd{52.54}{2.52} & 0.42 \\
NoT & \msd{86.83}{2.03} & \msd{95.67}{4.64} & \msd{88.29}{2.44} & \msd{44.59}{1.13} & 1.03 & \msdb{92.26}{2.82} & \msd{0.00}{0.00} & \msd{83.08}{2.12} & \msd{47.60}{1.62} & 3.24 & \msd{88.25}{3.52} & \msd{0.00}{0.00} & \msd{80.96}{3.42} & \msd{52.54}{2.23} & 0.45 \\
FedCARE & \msdb{88.12}{1.45} & \msd{96.54}{3.55} & \msd{89.36}{1.54} & \msd{42.93}{1.77} & 0.47 & \msd{89.89}{1.52} & \msd{0.00}{0.00} & \msd{81.71}{1.72} & \msd{50.72}{1.13} & 1.53 & \msd{91.65}{1.63} & \msd{7.94}{10.55} & \msd{83.30}{2.55} & \msd{51.01}{1.27} & 3.36 \\
FedMemPrune & \msd{87.22}{3.35} & \msd{96.10}{1.70} & \msd{88.16}{2.25} & \msd{44.75}{1.45} & 0.90 & \msd{69.80}{2.23} & \msd{0.00}{0.00} & \msd{62.65}{1.85} & \msd{45.54}{1.61} & 11.07 & \msd{72.64}{2.35} & \msd{0.00}{0.00} & \msd{65.44}{1.25} & \msd{50.46}{1.22} & 8.64 \\
\specialrule{0.45pt}{2pt}{1pt}
\rowcolor{black!5} \textbf{SURE (Ours)} & \msd{87.81}{4.45} & \msd{96.97}{1.77} & \msd{89.36}{4.07} & \msd{44.66}{1.35} & \textbf{0.21} & \msd{80.88}{2.05} & \msd{0.00}{0.00} & \msd{72.91}{2.36} & \msd{55.77}{1.64} & 4.19 & \msd{89.92}{3.21} & \msd{0.00}{0.00} & \msd{80.75}{3.27} & \msd{52.95}{2.26} & \textbf{0.29} \\
\specialrule{0.45pt}{1pt}{2pt}
\addlinespace[1pt]
\multicolumn{16}{c}{\textit{CIFAR-10 / CNN}} \\
\cmidrule(lr){1-16}
\textit{Retrain} & \msd{45.65}{2.19} & \msd{0.00}{0.00} & \msd{40.90}{1.39} & \msd{46.98}{1.97} & 0.00 & \msd{45.16}{1.32} & \msd{0.00}{0.00} & \msd{40.07}{0.59} & \msd{46.31}{0.46} & 0.00 & \msd{39.78}{3.24} & \msd{0.00}{0.00} & \msd{35.80}{1.09} & \msd{47.97}{1.52} & 0.00 \\
\cmidrule(lr){1-16}
Fine-tune & \msd{34.10}{2.81} & \msd{1.90}{0.12} & \msd{30.84}{1.72} & \msd{52.74}{3.62} & 7.32 & \msd{42.37}{1.62} & \msd{12.47}{3.82} & \msd{37.49}{1.24} & \msd{47.49}{2.32} & 4.76 & \msd{30.99}{1.52} & \msd{8.82}{1.73} & \msd{27.54}{1.62} & \msd{52.41}{2.26} & 7.58 \\
FedEraser & \msd{33.00}{2.09} & \msd{2.30}{0.09} & \msd{30.15}{1.28} & \msd{49.73}{3.75} & 7.11 & \msd{21.65}{2.25} & \msd{0.00}{0.00} & \msd{19.25}{1.62} & \msd{49.04}{2.62} & 11.77 & \msd{14.89}{2.26} & \msd{0.00}{0.00} & \msd{14.03}{1.72} & \msd{50.53}{3.62} & 12.31 \\
FedRecovery & \msd{36.47}{2.69} & \msd{54.50}{12.15} & \msd{38.24}{0.89} & \msd{53.77}{1.62} & 18.28 & \msd{36.94}{1.19} & \msd{51.00}{8.32} & \msdb{38.28}{0.23} & \msd{53.75}{1.24} & 17.11 & \msd{37.11}{2.99} & \msd{50.20}{8.55} & \msd{38.27}{1.82} & \msd{52.71}{1.62} & 15.02 \\
SalUn & \msd{27.33}{4.63} & \msd{0.00}{0.00} & \msd{24.87}{0.85} & \msd{46.36}{0.48} & 8.74 & \msd{11.11}{2.32} & \msd{0.00}{0.00} & \msd{10.00}{3.53} & \msd{49.00}{0.63} & 16.70 & \msd{14.25}{1.52} & \msd{0.00}{0.00} & \msd{18.54}{0.62} & \msd{52.11}{0.63} & 11.73 \\
MoDe & \msd{21.46}{3.06} & \msd{0.00}{0.00} & \msd{19.37}{1.93} & \msd{48.96}{1.42} & 11.93 & \msd{17.41}{3.62} & \msd{0.00}{0.00} & \msd{15.06}{3.31} & \msd{51.76}{0.73} & 14.55 & \msd{18.28}{1.38} & \msd{0.00}{0.00} & \msd{17.50}{1.23} & \msd{49.50}{0.73} & 10.33 \\
FedOSD & \msd{37.53}{3.40} & \msd{25.00}{4.06} & \msd{36.20}{2.15} & \msd{52.93}{1.92} & 10.94 & \msd{41.47}{3.26} & \msd{0.00}{0.00} & \msd{34.71}{2.62} & \msd{49.28}{1.45} & 3.01 & \msd{43.58}{2.30} & \msd{0.00}{0.00} & \msd{38.44}{1.93} & \msd{50.78}{1.31} & 2.31 \\
NoT & \msd{32.22}{1.79} & \msd{0.00}{0.00} & \msd{29.30}{1.58} & \msd{54.44}{2.41} & 8.12 & \msd{23.09}{1.61} & \msd{0.00}{0.00} & \msd{20.49}{2.23} & \msd{53.93}{1.73} & 12.32 & \msd{27.21}{1.72} & \msd{0.00}{0.00} & \msd{24.21}{1.28} & \msd{51.99}{2.62} & 7.05 \\
FedCARE & \msdb{40.81}{2.86} & \msd{20.10}{7.64} & \msdb{38.32}{2.34} & \msd{48.35}{1.14} & 7.22 & \msd{39.72}{1.89} & \msd{0.00}{0.00} & \msd{33.93}{3.62} & \msd{49.60}{1.24} & 3.72 & \msd{34.58}{2.26} & \msd{0.00}{0.00} & \msd{31.12}{1.62} & \msd{50.15}{1.36} & 3.02 \\
FedMemPrune & \msd{14.15}{0.97} & \msd{0.00}{0.00} & \msd{12.75}{0.37} & \msd{49.58}{2.75} & 15.56 & \msd{21.80}{0.62} & \msd{0.00}{0.00} & \msd{18.79}{0.72} & \msd{51.51}{2.85} & 12.46 & \msd{12.79}{1.62} & \msd{0.00}{0.00} & \msd{11.59}{0.32} & \msd{50.74}{2.73} & 13.49 \\
\specialrule{0.45pt}{2pt}{1pt}
\rowcolor{black!5} \textbf{SURE (Ours)} & \msd{37.14}{1.23} & \msd{3.70}{0.64} & \msd{34.15}{1.02} & \msd{51.74}{3.38} & \textbf{5.93} & \msdb{42.94}{1.57} & \msd{0.00}{0.00} & \msd{37.25}{1.62} & \msd{51.48}{2.37} & \textbf{2.55} & \msdb{44.22}{0.91} & \msd{0.00}{0.00} & \msdb{39.47}{0.82} & \msd{47.77}{2.31} & \textbf{2.08} \\
\specialrule{0.45pt}{1pt}{2pt}
\addlinespace[1pt]
\multicolumn{16}{c}{\textit{CIFAR-10 / ResNet-18}} \\
\cmidrule(lr){1-16}
\textit{Retrain} & \msd{62.59}{1.12} & \msd{5.60}{0.36} & \msd{56.36}{1.26} & \msd{50.94}{0.75} & 0.00 & \msd{65.31}{1.12} & \msd{0.00}{0.00} & \msd{58.90}{0.72} & \msd{50.85}{1.21} & 0.00 & \msd{63.54}{1.93} & \msd{0.00}{0.00} & \msd{56.19}{1.62} & \msd{50.82}{0.62} & 0.00 \\
\cmidrule(lr){1-16}
Fine-tune & \msd{55.32}{1.24} & \msd{31.90}{1.63} & \msd{51.40}{2.42} & \msd{44.84}{2.42} & 11.16 & \msd{51.16}{1.62} & \msd{8.10}{2.21} & \msd{45.77}{2.51} & \msd{48.70}{2.32} & 9.38 & \msd{63.12}{1.82} & \msd{15.25}{4.72} & \msd{56.12}{3.36} & \msd{50.64}{1.26} & 3.98 \\
FedEraser & \msd{49.04}{2.02} & \msd{37.00}{1.98} & \msd{47.51}{1.72} & \msd{51.36}{1.26} & 13.56 & \msd{17.63}{1.69} & \msd{0.00}{0.00} & \msd{15.87}{2.31} & \msd{47.04}{3.32} & 23.63 & \msd{16.95}{2.24} & \msd{0.00}{0.00} & \msd{14.64}{1.32} & \msd{50.59}{2.73} & 22.09 \\
FedRecovery & \msd{59.21}{1.24} & \msd{71.80}{7.53} & \msdb{60.97}{0.57} & \msd{48.47}{2.53} & 19.17 & \msd{59.48}{2.32} & \msd{68.70}{13.25} & \msdb{60.92}{1.25} & \msd{46.56}{1.34} & 20.21 & \msd{62.55}{3.36} & \msd{57.67}{6.59} & \msdb{60.94}{1.63} & \msd{47.28}{1.24} & 16.74 \\
SalUn & \msd{48.96}{2.21} & \msd{16.60}{0.61} & \msd{45.58}{1.24} & \msd{50.91}{0.27} & \textbf{8.86} & \msd{11.86}{3.23} & \msd{0.00}{0.00} & \msd{13.86}{1.63} & \msd{51.93}{0.68} & 24.89 & \msd{15.14}{2.62} & \msd{0.00}{0.00} & \msd{19.48}{3.23} & \msd{51.24}{0.82} & 21.38 \\
MoDe & \msd{47.68}{0.93} & \msd{0.90}{0.30} & \msd{43.20}{1.73} & \msd{45.16}{3.62} & 9.64 & \msd{18.15}{3.61} & \msd{0.00}{0.00} & \msd{21.58}{1.25} & \msd{50.00}{0.62} & 21.33 & \msd{25.14}{2.22} & \msd{0.00}{0.00} & \msd{27.32}{1.83} & \msd{52.35}{1.62} & 17.20 \\
FedOSD & \msd{59.14}{2.30} & \msd{51.60}{1.83} & \msd{58.05}{2.35} & \msd{43.84}{1.28} & 14.56 & \msd{60.22}{2.43} & \msd{0.00}{0.00} & \msd{54.79}{2.92} & \msd{48.70}{1.32} & 2.84 & \msdb{64.43}{3.63} & \msd{0.03}{0.05} & \msd{56.88}{1.73} & \msd{47.09}{1.73} & 1.34 \\
NoT & \msd{50.79}{1.25} & \msd{21.50}{4.66} & \msd{44.96}{1.32} & \msd{48.14}{1.21} & 10.48 & \msd{52.69}{1.36} & \msd{0.00}{0.00} & \msd{46.75}{1.02} & \msd{47.27}{2.21} & 7.09 & \msd{47.07}{1.62} & \msd{0.00}{0.00} & \msd{42.03}{1.83} & \msd{47.65}{1.52} & 8.45 \\
FedCARE & \msdb{59.51}{1.89} & \msd{65.70}{1.27} & \msd{60.12}{1.32} & \msd{48.57}{1.84} & 17.33 & \msd{61.53}{1.24} & \msd{0.70}{0.64} & \msd{54.74}{2.25} & \msd{49.12}{1.63} & 2.59 & \msd{63.90}{1.32} & \msd{0.37}{0.24} & \msd{56.53}{1.84} & \msd{47.78}{1.16} & \textbf{1.03} \\
FedMemPrune & \msd{26.52}{1.17} & \msd{3.30}{0.21} & \msd{24.24}{0.73} & \msd{49.32}{2.32} & 18.03 & \msd{32.32}{0.62} & \msd{0.00}{0.00} & \msd{28.62}{0.62} & \msd{53.75}{2.21} & 16.54 & \msd{23.12}{1.94} & \msd{0.00}{0.00} & \msd{21.49}{0.45} & \msd{50.02}{2.23} & 18.98 \\
\specialrule{0.45pt}{2pt}{1pt}
\rowcolor{black!5} \textbf{SURE (Ours)} & \msd{52.94}{1.97} & \msd{27.80}{1.68} & \msd{49.32}{1.26} & \msd{48.64}{2.02} & 10.30 & \msdb{62.40}{1.23} & \msd{0.40}{0.35} & \msd{56.30}{0.32} & \msd{49.71}{2.15} & \textbf{1.76} & \msd{58.21}{1.62} & \msd{0.47}{0.43} & \msd{51.72}{0.84} & \msd{47.96}{2.51} & 3.28 \\
\specialrule{0.45pt}{1pt}{2pt}
\addlinespace[1pt]
\multicolumn{16}{c}{\textit{CIFAR-100 / ResNet-18}} \\
\cmidrule(lr){1-16}
\textit{Retrain} & \msd{67.63}{0.35} & \msd{48.67}{0.78} & \msd{44.85}{0.47} & \msd{63.30}{0.53} & 0.00 & \msd{72.54}{0.96} & \msd{0.00}{0.00} & \msd{42.54}{0.20} & \msd{52.34}{0.25} & 0.00 & \msd{69.73}{0.24} & \msd{0.00}{0.00} & \msd{40.36}{0.67} & \msd{55.46}{0.68} & 0.00 \\
\cmidrule(lr){1-16}
Fine-tune & \msd{68.02}{0.72} & \msd{54.89}{1.42} & \msd{45.04}{0.24} & \msd{66.11}{0.77} & 2.40 & \msd{56.30}{0.67} & \msd{11.11}{0.81} & \msd{36.48}{0.62} & \msd{60.30}{0.50} & 10.34 & \msd{57.06}{0.84} & \msd{11.11}{0.32} & \msd{38.02}{0.89} & \msd{59.47}{0.39} & 7.53 \\
FedEraser & \msd{43.43}{0.24} & \msd{29.33}{0.90} & \msd{32.14}{0.57} & \msd{58.94}{0.84} & 15.15 & \msd{22.64}{1.03} & \msd{0.00}{0.00} & \msd{19.41}{0.72} & \msd{49.38}{0.47} & 19.00 & \msd{10.84}{0.22} & \msd{0.00}{0.00} & \msd{8.93}{0.95} & \msd{47.11}{0.39} & 24.67 \\
FedRecovery & \msd{66.62}{0.82} & \msd{58.44}{1.29} & \msd{44.04}{0.46} & \msd{65.64}{1.15} & 3.48 & \msd{66.86}{0.53} & \msd{55.56}{5.18} & \msdb{43.99}{0.68} & \msd{65.22}{0.28} & 18.89 & \msd{66.81}{0.69} & \msd{55.33}{1.35} & \msdb{43.96}{0.23} & \msd{65.60}{0.92} & 18.00 \\
SalUn & \msd{61.78}{0.95} & \msd{45.56}{1.14} & \msd{41.76}{0.25} & \msd{61.49}{0.62} & 3.47 & \msd{12.91}{0.36} & \msd{0.00}{0.00} & \msd{10.64}{0.39} & \msd{47.81}{0.83} & 24.02 & \msd{3.41}{0.43} & \msd{0.00}{0.00} & \msd{3.13}{0.74} & \msd{46.93}{1.31} & 28.02 \\
MoDe & \msd{52.69}{0.84} & \msd{35.78}{2.32} & \msd{37.51}{0.34} & \msd{59.91}{0.51} & 9.64 & \msd{23.48}{0.27} & \msd{0.00}{0.00} & \msd{19.63}{0.65} & \msd{53.61}{0.74} & 18.31 & \msd{15.63}{0.22} & \msd{0.00}{0.00} & \msd{13.08}{0.31} & \msd{52.30}{0.75} & 21.14 \\
FedOSD & \msd{66.99}{0.29} & \msd{56.22}{1.34} & \msd{43.90}{0.83} & \msd{65.11}{0.72} & 2.74 & \msdb{70.57}{0.50} & \msd{3.56}{0.54} & \msd{42.19}{0.59} & \msd{56.70}{0.32} & 2.56 & \msdb{70.57}{0.41} & \msd{4.22}{1.33} & \msd{42.16}{0.75} & \msd{57.54}{0.30} & 2.24 \\
NoT & \msd{57.23}{0.31} & \msd{42.00}{0.38} & \msd{40.80}{0.56} & \msd{60.70}{0.94} & 5.93 & \msd{56.64}{0.76} & \msd{6.22}{0.45} & \msd{37.00}{0.31} & \msd{58.62}{0.80} & 8.49 & \msd{56.02}{0.62} & \msd{4.67}{1.18} & \msd{39.48}{0.79} & \msd{58.83}{0.95} & 5.66 \\
FedCARE & \msd{67.31}{0.55} & \msd{56.89}{1.75} & \msd{44.44}{0.19} & \msd{65.79}{1.52} & 2.86 & \msd{69.14}{0.42} & \msd{4.22}{0.72} & \msd{42.13}{0.74} & \msd{54.32}{0.68} & \textbf{2.50} & \msd{68.96}{0.95} & \msd{4.22}{1.45} & \msd{42.26}{0.83} & \msd{54.26}{0.81} & \textbf{2.02} \\
FedMemPrune & \msd{41.83}{0.24} & \msd{29.33}{3.42} & \msd{32.28}{1.25} & \msd{55.76}{0.30} & 16.31 & \msd{45.43}{0.78} & \msd{0.00}{0.00} & \msd{35.33}{0.23} & \msd{48.52}{0.51} & 9.54 & \msd{36.54}{0.97} & \msd{0.00}{0.00} & \msd{28.45}{0.89} & \msd{47.53}{0.57} & 13.26 \\
\specialrule{0.45pt}{2pt}{1pt}
\rowcolor{black!5} \textbf{SURE (Ours)} & \msdb{68.07}{0.63} & \msd{52.89}{1.33} & \msdb{45.09}{0.78} & \msd{65.16}{0.85} & \textbf{1.69} & \msd{70.54}{0.64} & \msd{4.89}{0.43} & \msd{38.69}{0.39} & \msd{59.51}{1.07} & 4.48 & \msd{60.02}{1.15} & \msd{4.89}{1.07} & \msd{35.82}{0.72} & \msd{56.56}{0.53} & 5.06 \\
\bottomrule
\end{tabular}%
}
\endgroup
\vspace{-1.5em}
\end{table*}

\begin{table*}[t]
\centering
\caption{Class-wise unlearning under Dirichlet non-IID partitions ($\alpha=0.1$). Ret., For., Test, and MIA are mean $\pm$ SD (\%). AG (Avg. Gap)$\downarrow$ is the mean absolute difference from mode-matched Retrain (pp), computed from reported metric means. Ret./Test$\uparrow$: higher is better; bold highlights the best non-reference Ret./Test means and lowest AG.
}
\label{tab:classwise-main}
\begingroup
\small
\setlength{\tabcolsep}{1pt}
\renewcommand{\arraystretch}{0.1}
\resizebox{\textwidth}{!}{%
\begin{tabular}{@{}l*{15}{c}@{}}
\toprule
& \multicolumn{5}{c}{Client-only} & \multicolumn{5}{c}{Server-only} & \multicolumn{5}{c}{Joint} \\
\cmidrule(lr){2-6}\cmidrule(lr){7-11}\cmidrule(l){12-16}
Method & Ret.$\uparrow$ & For. & Test$\uparrow$ & MIA & AG$\downarrow$ & Ret.$\uparrow$ & For. & Test$\uparrow$ & MIA & AG$\downarrow$ & Ret.$\uparrow$ & For. & Test$\uparrow$ & MIA & AG$\downarrow$ \\
\midrule
\multicolumn{16}{c}{\textit{MNIST / LeNet}} \\
\cmidrule(lr){1-16}
\textit{Retrain} & \msd{99.23}{0.02} & \msd{99.43}{0.25} & \msd{98.74}{0.21} & \msd{49.46}{4.13} & 0.00 & \msd{99.11}{0.10} & \msd{0.00}{0.00} & \msd{88.96}{0.35} & \msd{50.33}{1.44} & 0.00 & \msd{99.12}{0.14} & \msd{0.00}{0.00} & \msd{89.01}{0.31} & \msd{49.34}{3.14} & 0.00 \\
\cmidrule(lr){1-16}
Fine-tune & \msd{99.19}{0.06} & \msd{99.43}{0.25} & \msd{97.81}{0.23} & \msd{48.99}{0.77} & 0.36 & \msdb{99.52}{0.10} & \msd{92.67}{1.14} & \msd{98.34}{0.23} & \msd{46.79}{3.92} & 26.50 & \msdb{99.55}{0.06} & \msd{93.10}{1.55} & \msd{98.37}{0.22} & \msd{48.57}{2.26} & 25.92 \\
SalUn & \msd{97.92}{0.30} & \msd{86.93}{2.49} & \msd{96.32}{0.33} & \msd{47.09}{5.48} & 4.65 & \msd{98.37}{0.94} & \msd{7.47}{2.17} & \msd{88.91}{0.47} & \msd{50.47}{0.21} & 2.10 & \msd{99.38}{0.26} & \msd{7.04}{2.17} & \msd{89.61}{0.54} & \msd{50.24}{0.21} & 2.20 \\
MoDe & \msd{99.18}{0.10} & \msd{99.57}{0.43} & \msd{98.71}{0.27} & \msd{47.75}{1.84} & 0.48 & \msd{96.46}{0.44} & \msd{0.00}{0.00} & \msd{86.93}{0.88} & \msd{48.75}{2.36} & 1.57 & \msd{96.70}{0.46} & \msd{0.00}{0.00} & \msd{87.01}{0.97} & \msd{47.68}{1.66} & 1.52 \\
FedOSD & \msd{99.09}{0.07} & \msd{99.57}{0.43} & \msdb{98.83}{0.32} & \msd{47.25}{3.82} & 0.65 & \msd{99.14}{0.06} & \msd{99.28}{0.50} & \msdb{98.83}{0.28} & \msd{50.70}{3.59} & 27.39 & \msd{99.14}{0.06} & \msd{99.28}{0.50} & \msdb{98.78}{0.29} & \msd{50.11}{4.06} & 27.46 \\
NoT & \msd{97.88}{0.41} & \msd{96.26}{3.03} & \msd{97.30}{1.13} & \msd{47.54}{4.20} & 1.97 & \msd{98.68}{0.18} & \msd{0.00}{0.00} & \msd{88.68}{0.33} & \msd{51.77}{2.59} & 0.54 & \msd{98.86}{0.30} & \msd{1.87}{2.21} & \msd{88.98}{0.50} & \msd{49.48}{4.11} & 0.58 \\
FedCARE & \msdb{99.22}{0.01} & \msd{99.43}{0.25} & \msd{98.80}{0.17} & \msd{48.45}{3.92} & 0.27 & \msdb{99.52}{0.21} & \msd{87.81}{2.02} & \msd{93.83}{0.26} & \msd{48.04}{4.51} & 23.85 & \msd{99.34}{0.34} & \msd{91.81}{3.42} & \msd{98.23}{0.22} & \msd{48.00}{4.47} & 25.65 \\
FedMemPrune & \msd{99.18}{0.05} & \msd{99.28}{0.25} & \msd{98.20}{0.24} & \msd{49.27}{3.44} & 0.23 & \msd{99.13}{0.15} & \msd{1.29}{1.88} & \msd{89.08}{0.20} & \msd{48.31}{0.81} & 0.86 & \msd{99.11}{0.09} & \msd{0.86}{1.14} & \msd{88.99}{0.27} & \msd{52.32}{2.01} & 0.97 \\
\specialrule{0.45pt}{2pt}{1pt}
\rowcolor{black!5} \textbf{SURE (Ours)} & \msd{99.20}{0.03} & \msd{99.43}{0.25} & \msd{98.77}{0.27} & \msd{48.79}{1.30} & \textbf{0.18} & \msd{99.06}{0.15} & \msd{0.00}{0.00} & \msd{88.96}{0.28} & \msd{50.53}{1.19} & \textbf{0.06} & \msd{99.10}{0.11} & \msd{0.00}{0.00} & \msd{88.98}{0.29} & \msd{50.14}{1.75} & \textbf{0.21} \\
\specialrule{0.45pt}{1pt}{2pt}
\addlinespace[1pt]
\multicolumn{16}{c}{\textit{CIFAR-10 / ResNet-18}} \\
\cmidrule(lr){1-16}
\textit{Retrain} & \msd{81.14}{1.99} & \msd{33.67}{7.25} & \msd{72.96}{1.55} & \msd{50.77}{1.16} & 0.00 & \msd{81.29}{2.50} & \msd{0.00}{0.00} & \msd{69.45}{1.45} & \msd{51.10}{0.70} & 0.00 & \msd{81.18}{2.80} & \msd{0.00}{0.00} & \msd{69.55}{2.11} & \msd{51.82}{3.85} & 0.00 \\
\cmidrule(lr){1-16}
Fine-tune & \msd{80.49}{2.05} & \msd{41.90}{3.36} & \msd{70.84}{1.92} & \msd{53.70}{0.44} & 3.48 & \msd{81.42}{2.74} & \msd{0.00}{0.00} & \msd{69.44}{1.67} & \msd{49.82}{1.34} & \textbf{0.36} & \msd{83.13}{1.85} & \msd{0.00}{0.00} & \msd{71.12}{1.16} & \msd{51.43}{2.72} & 0.98 \\
SalUn & \msd{71.50}{7.87} & \msd{25.70}{9.05} & \msd{64.92}{5.06} & \msd{51.65}{2.79} & 6.63 & \msd{55.39}{8.54} & \msd{0.00}{0.00} & \msd{48.56}{6.92} & \msd{50.26}{2.86} & 11.91 & \msd{31.24}{11.74} & \msd{0.00}{0.00} & \msd{28.00}{10.32} & \msd{50.78}{1.68} & 23.13 \\
MoDe & \msd{71.90}{2.25} & \msd{28.77}{0.76} & \msd{65.54}{1.23} & \msd{51.22}{1.05} & 5.50 & \msd{47.19}{3.57} & \msd{0.00}{0.00} & \msd{41.14}{2.68} & \msd{53.33}{1.13} & 16.16 & \msd{29.51}{3.93} & \msd{0.00}{0.00} & \msd{26.59}{3.71} & \msd{51.21}{1.38} & 23.81 \\
FedOSD & \msd{78.56}{2.95} & \msd{41.63}{5.88} & \msd{70.38}{1.07} & \msd{53.82}{1.68} & 4.04 & \msd{82.03}{1.69} & \msd{0.00}{0.00} & \msd{70.34}{1.36} & \msd{51.31}{3.70} & 0.46 & \msd{82.40}{1.87} & \msd{0.00}{0.00} & \msd{70.64}{1.54} & \msd{51.69}{4.10} & \textbf{0.61} \\
NoT & \msd{73.89}{3.06} & \msd{35.03}{3.56} & \msd{68.40}{1.74} & \msd{50.21}{1.54} & 3.43 & \msd{77.82}{1.11} & \msd{0.00}{0.00} & \msd{65.76}{0.77} & \msd{50.97}{0.24} & 1.82 & \msd{83.59}{1.09} & \msd{0.00}{0.00} & \msdb{72.79}{0.48} & \msd{52.24}{2.27} & 1.52 \\
FedCARE & \msd{80.79}{2.19} & \msd{78.90}{2.99} & \msdb{77.17}{1.86} & \msd{51.16}{3.54} & 12.55 & \msd{82.11}{1.76} & \msd{15.30}{16.47} & \msdb{71.96}{2.89} & \msd{52.71}{2.42} & 5.06 & \msd{82.13}{1.79} & \msd{15.50}{16.44} & \msd{71.99}{2.87} & \msd{52.82}{2.69} & 4.97 \\
FedMemPrune & \msd{77.56}{3.15} & \msd{43.33}{10.66} & \msd{71.38}{1.77} & \msd{52.88}{3.13} & 4.23 & \msd{77.66}{4.88} & \msd{0.00}{0.00} & \msd{66.54}{3.53} & \msd{50.79}{1.65} & 1.71 & \msd{75.46}{4.27} & \msd{0.00}{0.00} & \msd{64.96}{3.82} & \msd{52.35}{0.27} & 2.71 \\
\specialrule{0.45pt}{2pt}{1pt}
\rowcolor{black!5} \textbf{SURE (Ours)} & \msdb{81.74}{1.83} & \msd{41.60}{5.24} & \msd{73.99}{1.84} & \msd{53.74}{1.27} & \textbf{3.13} & \msdb{83.87}{1.85} & \msd{0.23}{0.40} & \msd{71.24}{1.07} & \msd{52.73}{1.71} & 1.56 & \msdb{83.72}{1.44} & \msd{0.10}{0.17} & \msd{71.33}{1.08} & \msd{51.66}{0.33} & 1.15 \\
\bottomrule
\end{tabular}%
}
\endgroup
\end{table*}

\section{Experimental Evaluation}
\label{sec:experiments}


We systematically evaluate SURE and adapted unlearning methods under SFU, examining forgetting--utility trade-offs, model partitioning, and system efficiency.

\subsection{Experimental Setup}
\label{sec:exp_setup}

\paragraph{Datasets and configurations.}
We consider two unlearning scenarios: \emph{client-wise unlearning}, which removes the contributions of a designated client, and \emph{class-wise unlearning}, which removes a target class across clients. Experiments cover three datasets and four dataset--model configurations: MNIST/LeNet, CIFAR-10/CNN, CIFAR-10/ResNet-18, and CIFAR-100/ResNet-18. For client-wise unlearning, we consider ten clients with one removal request under Independent and Identically Distributed (IID) and extreme non-IID data partitions, the latter assigning disjoint label subsets to clients. 
For class-wise unlearning, we remove a designated class across clients under Dirichlet-based non-IID partitions ($\alpha=0.1$)~\cite{hsu2019measuring}.
Both scenarios are examined under Client-only, Server-only, and Joint update modes. ResNet-18 is partitioned after its first residual block, with additional cut positions examined separately. Further details on data partitioning, model architectures, and training configurations are provided in Appendix~\ref{app:experimental-settings}.

\paragraph{Baselines and retraining references.}
We compare SURE with retained-data fine-tuning (Fine-tune) and eight unlearning methods adapted to SFL: FedEraser~\cite{liu2021federaser}, FedRecovery~\cite{zhang2023fedrecovery}, SalUn~\cite{fan2024salun}, MoDe~\cite{zhao2024mode}, FedOSD~\cite{pan2025fedosd}, NoT~\cite{khalil2025not}, FedCARE~\cite{li2026fedcare}, and FedMemPrune~\cite{wei2026fedmemprune}. All methods follow the same SFL architecture and model update constraints, with adaptations detailed in Appendix~\ref{app:sfl-baseline-adaptation}. We use mode-matched Retrain as the evaluation reference. For Client-only and Server-only, the updatable partition is reinitialized and retrained on retained data, while the other partition remains fixed at its trained state. Joint Retrain reinitializes and retrains both partitions on retained data, corresponding to the SFL retraining reference defined in Section~\ref{sec:problem_formulation}. Further details are provided in Appendix~\ref{app:experimental-settings} and~\ref{app:sfl-baseline-adaptation}.


\paragraph{Evaluation Protocol.}
We report retained-training accuracy (Retain), forgetting-target accuracy (Forget), global test accuracy (Test), and membership-inference ROC-AUC (MIA). Forget is evaluated on the removed client's training samples for client-wise unlearning and target-class test samples for class-wise unlearning. Agreement with mode-matched Retrain is measured by
$
\mathrm{Avg.\ Gap}=
\frac{1}{4}\sum_{q\in\mathcal{J}}
|q(\theta^U)-q(\theta_{\mathcal{M}}^R)|,
\label{eq:experimental-score}
$
where $\mathcal{J}=\{\mathrm{Retain},\mathrm{Forget},\mathrm{Test},\mathrm{MIA}\}$ and $\theta_{\mathcal{M}}^R$ is the mode-matched retraining reference. Lower gaps indicate closer agreement, measured in percentage points. We also report communication (bytes) and computation (FLOPs). Results are averaged over three runs with standard deviations, using at most 40 recovery rounds for SURE. Further evaluation details appear in Appendices~\ref{app:mia} and~\ref{app:metric-cost-protocol}.

\begin{figure*}[t]
\centering
\includegraphics[width=\textwidth]{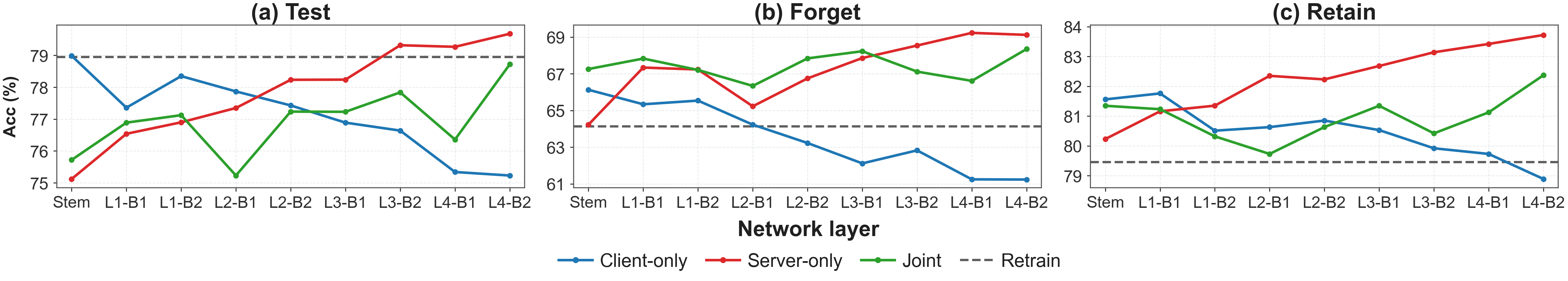}
\setlength{\abovecaptionskip}{-5pt}

\caption{Impact of split position on client-wise unlearning with SURE on CIFAR-10/ResNet-18. Panels show (a) Test, (b) Forget, and (c) Retain accuracy across nine split positions. Colored curves represent Client-only, Server-only, and Joint updates, while dashed lines indicate mode-matched Retrain references.}

\label{fig:sure-split-performance}
\vspace{-1.0em}
\end{figure*}

\begin{figure*}[t]
\centering
\includegraphics[width=\textwidth]{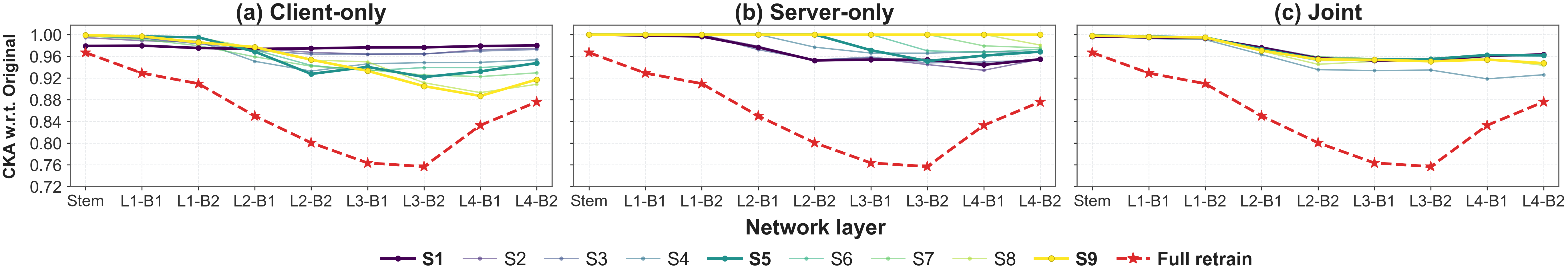}
\setlength{\abovecaptionskip}{-5pt}
\caption{Layer-wise CKA across nine split positions for client-wise unlearning on CIFAR-10/ResNet-18. The three panels correspond to Client-only, Server-only, and Joint updates. Colored curves represent different split positions, while the red dashed curve denotes Retrain. All CKA values are computed relative to Original.}
\label{fig:sure-split-cka}
\vspace{-1.1em}
\end{figure*}

\providecommand{\msd}[2]{\ensuremath{#1_{\scriptscriptstyle\pm #2}}}
\providecommand{\msdb}[2]{\ensuremath{\mathbf{#1}_{\scriptscriptstyle\pm #2}}}

\begin{table*}[t]
\centering
\caption{SURE ablation on CIFAR-10/ResNet-18 for client-wise unlearning (Dirichlet $\alpha=0.1$, split 2). 
A0 denotes full SURE; A1--A7 are described in Section~\ref{sec:exp_ablation}.
 Results are mean $\pm$ SD (\%). Ret. and Test ($\uparrow$) denote predictive utility, while For. and MIA are interpreted relative to mode-matched Retrain. AG (Avg. Gap) ($\downarrow$, pp) is the mean absolute difference from Retrain across the four metrics. Bold marks the highest non-Retrain utility or lowest gap.}
\label{tab:sure-ablation}
\begingroup
\small
\setlength{\tabcolsep}{1pt}
\renewcommand{\arraystretch}{0.1}
\resizebox{\textwidth}{!}{%
\begin{tabular}{@{}l*{15}{c}@{}}
\toprule
 & \multicolumn{5}{c}{Client-only}
 & \multicolumn{5}{c}{Server-only}
 & \multicolumn{5}{c}{Joint} \\
\cmidrule(lr){2-6}
\cmidrule(lr){7-11}
\cmidrule(l){12-16}
Method
& Ret.$\uparrow$ & For. & Test$\uparrow$ & MIA & AG$\downarrow$
& Ret.$\uparrow$ & For. & Test$\uparrow$ & MIA & AG$\downarrow$
& Ret.$\uparrow$ & For. & Test$\uparrow$ & MIA & AG$\downarrow$ \\
\midrule
\textit{Retrain}
& \msd{80.67}{1.96} & \msd{68.14}{1.08} & \msd{76.38}{1.67} & \msd{50.60}{2.01} & 0.00
& \msd{79.67}{1.48} & \msd{65.37}{3.33} & \msd{75.11}{0.82} & \msd{49.67}{1.24} & 0.00
& \msd{79.46}{2.47} & \msd{64.13}{2.32} & \msd{75.77}{1.86} & \msd{51.12}{0.23} & 0.00 \\
\cmidrule(lr){1-16}
A1
& \msd{80.99}{2.46} & \msd{66.00}{4.95} & \msd{76.53}{1.84} & \msd{51.13}{1.35} & 0.79
& \msd{81.38}{1.20} & \msd{71.80}{4.75} & \msd{77.00}{1.34} & \msd{50.51}{4.29} & 2.72
& \msd{80.06}{5.67} & \msd{69.74}{1.42} & \msd{75.86}{3.97} & \msd{49.22}{3.29} & 2.05 \\
A2
& \msdb{81.23}{2.50} & \msd{67.83}{3.74} & \msdb{76.89}{2.09} & \msd{50.87}{2.15} & \textbf{0.41}
& \msd{81.03}{2.31} & \msd{70.93}{1.17} & \msd{76.12}{1.59} & \msd{49.54}{1.05} & 2.02
& \msdb{82.20}{2.46} & \msd{70.59}{2.44} & \msd{77.35}{1.36} & \msd{52.12}{1.86} & 2.95 \\
A3
& \msd{79.40}{1.48} & \msd{65.24}{3.92} & \msd{75.77}{1.72} & \msd{50.24}{2.92} & 1.29
& \msd{79.62}{5.68} & \msd{70.61}{3.16} & \msd{74.70}{4.84} & \msd{51.28}{2.75} & 1.83
& \msd{79.81}{3.19} & \msd{70.60}{2.71} & \msd{75.87}{2.44} & \msd{49.71}{2.95} & 2.08 \\
A4
& \msd{78.75}{1.55} & \msd{64.91}{4.49} & \msd{75.20}{2.00} & \msd{51.05}{2.10} & 1.70
& \msd{81.68}{2.62} & \msd{73.00}{3.31} & \msd{76.77}{1.49} & \msd{50.47}{1.20} & 3.03
& \msd{82.13}{1.54} & \msd{71.21}{3.82} & \msd{77.37}{1.18} & \msd{49.74}{3.21} & 3.18 \\
A5
& \msd{80.87}{2.66} & \msd{66.25}{2.29} & \msd{76.16}{1.99} & \msd{51.22}{2.11} & 0.73
& \msd{81.98}{2.05} & \msd{71.03}{1.82} & \msdb{77.03}{1.48} & \msd{49.45}{2.77} & 2.53
& \msd{81.40}{2.37} & \msd{66.85}{2.12} & \msd{76.99}{2.14} & \msd{48.75}{2.96} & 2.06 \\
A6
& \msd{80.97}{2.72} & \msd{66.43}{3.39} & \msd{76.36}{2.20} & \msd{50.21}{1.67} & 0.61
& \msdb{82.04}{2.23} & \msd{71.20}{5.40} & \msd{76.78}{0.84} & \msd{51.21}{1.94} & 2.85
& \msd{76.77}{5.58} & \msd{65.41}{1.47} & \msd{72.42}{3.77} & \msd{49.29}{1.84} & 2.29 \\
A7
& \msd{80.42}{2.08} & \msd{67.03}{3.83} & \msd{76.36}{2.00} & \msd{52.96}{0.09} & 0.94
& \msd{78.31}{4.46} & \msd{70.27}{4.50} & \msd{73.75}{2.96} & \msd{50.98}{0.47} & 2.23
& \msdb{82.20}{1.63} & \msd{70.11}{2.01} & \msdb{77.40}{1.02} & \msd{51.80}{1.11} & 2.76 \\
\specialrule{0.45pt}{2pt}{1pt}
\rowcolor{black!5}
\textbf{A0 (SURE)}
& \msdb{81.23}{2.50} & \msd{67.83}{3.74} & \msdb{76.89}{2.09} & \msd{50.87}{2.15} & \textbf{0.41}
& \msd{81.16}{1.99} & \msd{67.43}{4.39} & \msd{76.54}{4.38} & \msd{51.22}{1.33} & \textbf{1.63}
& \msd{81.76}{2.17} & \msd{65.34}{4.93} & \msd{77.36}{1.22} & \msd{48.61}{1.94} & \textbf{1.90} \\
\bottomrule
\end{tabular}%
}
\endgroup
\vspace{-0.3em}
\end{table*}

\subsection{Unlearning Performance Comparison}
\label{sec:exp_comparison}


\textbf{\textit{Client-Wise Unlearning.}}
Under extreme non-IID partitions, adapted methods exhibit substantial forgetting--utility trade-offs (Table~\ref{tab:clientwise-extreme}). On CIFAR-10/ResNet-18 with Server-only updates, FedRecovery retains 68.70\% Forget accuracy, while SalUn and MoDe reduce Test accuracy to 13.86\% and 21.58\%, respectively. SURE achieves 0.40\% Forget, 56.30\% Test, and an Avg. Gap of 1.76 percentage points (pp) relative to Retrain. Under Joint updates on CIFAR-10/CNN, SURE attains the highest non-reference Retain (44.22\%) and Test (39.47\%) accuracy, with 0.00\% Forget and the lowest Avg. Gap (2.08 pp). Results on CIFAR-100/ResNet-18 reveal remaining performance gaps, underscoring the influence of model architecture and update mode. Additional IID and multiple-client removal results are provided in Appendices~\ref{app:additional-iid} (Table~\ref{tab:clientwise-iid}) and~\ref{app:additional-multiclient} (Table~\ref{tab:sure-many}), respectively.


\textbf{\textit{Class-Wise Unlearning.}}
Class-wise removal similarly exposes challenges in balancing forgetting and retained utility (Table~\ref{tab:classwise-main}). On MNIST/LeNet with Server-only updates, Fine-tune and FedOSD retain high target-class accuracy (92.67\% and 99.28\%), whereas SURE matches Retrain's 0.00\% Forget and 88.96\% Test accuracy, achieving the lowest Avg. Gap (0.06 pp). On CIFAR-10/ResNet-18, SURE achieves the highest Retain accuracy across all three modes. Under Joint updates, it maintains 83.72\% Retain and 71.33\% Test accuracy with 0.10\% Forget and an Avg. Gap of 1.15 pp, although FedOSD achieves a smaller gap (0.61 pp). These findings demonstrate competitive retraining agreement while highlighting the mode-dependent limitations of unlearning in SFL.

\subsection{Impact of Split Position on Unlearning}
\label{sec:exp_partitioning}


Split position determines the trainable model partitions under different update modes, thereby influencing unlearning behavior. We evaluate SURE across nine ResNet-18 split positions (Fig.~\ref{fig:sure-split-performance}). As the split deepens, Client-only updates exhibit declining Test and Retain accuracy, whereas Server-only updates show the opposite trend. Joint updates exhibit non-monotonic behavior, while Forget accuracy varies across modes. To examine the representation changes, Fig.~\ref{fig:sure-split-cka} reports layer-wise Centered Kernel Alignment (CKA)~\cite{kornblith2019similarity} relative to Original. Final-stage CKA decreases from 0.980 to 0.908 between splits 1 and 8 under Client-only updates, but increases from 0.954 to 0.980 under Server-only updates. Additional t-SNE visualizations in Appendix~\ref{app:additional-representations} (Fig.~\ref{fig:sure-clientwise-tsne} and Fig.~\ref{fig:sure-cifar-rep}) illustrate representation differences across split positions and model partitions. These results highlight the interplay between split placement and update constraints, although representation changes alone do not establish forgetting effectiveness.

\begin{table}[t]
\centering
\caption{Communication and computation costs for client-wise unlearning under extreme non-IID partitions. Communication is in GB ($10^9$ bytes) and computation in TFLOPs ($10^{12}$ operations); lower is better.
}
\label{tab:sfu-efficiency}
\begingroup
\small
\setlength{\tabcolsep}{3pt}
\renewcommand{\arraystretch}{0.3}
\begin{tabular}{@{}l*{6}{c}@{}}
\toprule
 & \multicolumn{2}{c}{Client-only} & \multicolumn{2}{c}{Server-only} & \multicolumn{2}{c}{Joint} \\
\cmidrule(lr){2-3}\cmidrule(lr){4-5}\cmidrule(l){6-7}
Method & Comm.$\downarrow$ & FLOPs$\downarrow$ & Comm.$\downarrow$ & FLOPs$\downarrow$ & Comm.$\downarrow$ & FLOPs$\downarrow$ \\
\midrule
\multicolumn{7}{c}{\textit{CIFAR-10 / CNN}} \\
\cmidrule(lr){1-7}
Fine-tune & 53.3 & 333 & 26.5 & 218 & 53.3 & 379 \\
FedEraser & 53.5 & 333 & 26.5 & 218 & 53.5 & 379 \\
FedRecovery & 0.189 & 0.00707 & 0.161 & 0.00707 & 0.189 & 0.00707 \\
SalUn & 59.1 & 371 & 29.6 & 244 & 59.1 & 422 \\
MoDe & 112 & 703 & 56.0 & 461 & 112 & 800 \\
FedOSD & 53.3 & 333 & 26.5 & 218 & 53.3 & 379 \\
NoT & 53.3 & 333 & 26.5 & 218 & 53.3 & 379 \\
FedCARE & 53.4 & 335 & 26.6 & 220 & 53.4 & 381 \\
FedMemPrune & 159 & 999 & 79.6 & 655 & 159 & 1{,}140 \\
\specialrule{0.45pt}{2pt}{1pt}
\rowcolor{black!5} \textbf{SURE (Ours)} & 54.7 & 344 & 28.0 & 229 & 54.7 & 390 \\
\midrule
\multicolumn{7}{c}{\textit{CIFAR-10 / ResNet-18}} \\
\cmidrule(lr){1-7}
Fine-tune & 213 & 980 & 106 & 1{,}233 & 213 & 1{,}370 \\
FedEraser & 213 & 980 & 106 & 1{,}233 & 213 & 1{,}370 \\
FedRecovery & 0.218 & 0.274 & 0.183 & 0.274 & 0.218 & 0.274 \\
SalUn & 237 & 1{,}093 & 119 & 1{,}373 & 237 & 1{,}525 \\
MoDe & 449 & 2{,}070 & 224 & 2{,}603 & 449 & 2{,}890 \\
FedOSD & 213 & 980 & 106 & 1{,}233 & 213 & 1{,}370 \\
NoT & 213 & 980 & 106 & 1{,}233 & 213 & 1{,}370 \\
FedCARE & 213 & 986 & 106 & 1{,}240 & 213 & 1{,}375 \\
FedMemPrune & 637 & 2{,}941 & 319 & 3{,}699 & 637 & 4{,}106 \\
\specialrule{0.45pt}{2pt}{1pt}
\rowcolor{black!5} \textbf{SURE (Ours)} & 218 & 1{,}020 & 112 & 1{,}270 & 218 & 1{,}370 \\
\bottomrule
\end{tabular}%
\endgroup
\vspace{-1em}
\end{table}

\subsection{Ablation Study}
\label{sec:exp_ablation}


We evaluate seven ablation variants (Table~\ref{tab:sure-ablation}): A1 removes magnitude weighting; A2 fixes the first trainable layer; A3 disables retained-importance protection ($\gamma=0$); A4 applies hard layer negation; 
A5 uses standard retained-data FedAvg recovery; and A6/A7 retain only forgotten-importance suppression or retained-importance amplification, respectively. Full SURE achieves the lowest Avg. Gap under Server-only (1.63 pp) and Joint (1.90 pp) updates, tying with A2 under Client-only (0.41 pp). A1 and A4 consistently increase the gap, while recovery variants reveal distinct forgetting--utility trade-offs. Under Joint updates, A6 reduces Test accuracy to 72.42\%, whereas A7 increases Forget accuracy to 70.11\%. In contrast, SURE achieves 77.36\% Test and 65.34\% Forget accuracy, approaching Retrain's 64.13\% Forget accuracy. These results highlight the complementary roles of importance estimation, selective perturbation, and guided recovery. 
Hyperparameter sensitivity is further examined in Appendix~\ref{app:additional-sensitivity} (Figure~\ref{fig:sure-sensitivity}).

\subsection{Efficiency Analysis}
\label{sec:exp_efficiency}


We compare communication and computation costs across three update modes (Table~\ref{tab:sfu-efficiency}). SURE incurs modest overhead over Fine-tune while remaining substantially more efficient than MoDe and FedMemPrune. On CIFAR-10/ResNet-18 under Joint updates, SURE requires 218 GB and 1,370 TFLOPs, compared with 213 GB and 1,370 TFLOPs for Fine-tune and 449 GB and 2,890 TFLOPs for MoDe. FedRecovery achieves substantially lower costs through training-free correction using stored historical updates, although historical storage costs are excluded~\cite{zhang2023fedrecovery}. However, lower costs do not necessarily imply closer agreement with Retrain. 

\section{Limitations}

SURE provides an initial approach to SFU, with effectiveness varying across model update modes, split positions, and forgetting scenarios. Its recovery stage requires retained data, and formal forgetting guarantees remain unexplored. Moreover, agreement with mode-matched retraining under restricted updates does not necessarily imply equivalence to full retraining. Future work will explore more scalable unlearning methods, broader SFL architectures, and stronger forgetting guarantees.

\section{Conclusion}



We establish \emph{Split Federated Unlearning} (SFU) as a distinct research problem, providing a unified formulation and theoretical characterization of historical influence propagation and mode-dependent unlearning feasibility. As an initial solution, SURE integrates dual-importance estimation, selective perturbation, and guided recovery. Extensive experiments across diverse model architectures, data distributions, and update modes reveal the forgetting--utility challenges of existing methods adapted to SFL, while SURE achieves competitive agreement with mode-matched retraining at modest overhead, laying a foundation for principled unlearning in split federated systems.

\bibliographystyle{ACM-Reference-Format}
\bibliography{references}

\appendix

\section{Additional Experimental Results}
\label{app:additional-experimental-results}

\subsection{Client-Wise Unlearning under IID Partitions}
\label{app:additional-iid}


Under IID partitions, retained clients share the same label space, so forgotten-sample accuracy may remain high even after retraining. As shown in Table~\ref{tab:clientwise-iid}, SURE preserves competitive utility across architectures and update modes. On CIFAR-10/ResNet-18 under Joint updates, SURE achieves 88.01\% Test accuracy, close to Retrain (87.60\%), whereas FedEraser and MoDe drop to 57.83\% and 46.78\%. Nevertheless, deviations in Forget accuracy and MIA, particularly on CIFAR-100, highlight the distinction between utility preservation and retraining agreement.

\providecommand{\msd}[2]{\ensuremath{#1_{\scriptscriptstyle\pm #2}}}
\providecommand{\msdb}[2]{\ensuremath{\mathbf{#1}_{\scriptscriptstyle\pm #2}}}
\begin{table*}[t]
\centering
\caption{Client-wise unlearning under IID partitions (mean $\pm$ SD, \%). AG (Avg. Gap)$\downarrow$ is the mean absolute distance of Ret., For., Test, and MIA from mode-matched Retrain (percentage points), computed from reported means. Ret./Test$\uparrow$: higher is better. Bold marks the best non-reference Ret./Test or AG within each dataset--model and update mode. 
}
\label{tab:clientwise-iid}
\begingroup
\small
\setlength{\tabcolsep}{1pt}
\renewcommand{\arraystretch}{0.1}
\resizebox{\textwidth}{!}{%
\begin{tabular}{@{}l*{15}{c}@{}}
\toprule
& \multicolumn{5}{c}{Client-only} & \multicolumn{5}{c}{Server-only} & \multicolumn{5}{c}{Joint} \\
\cmidrule(lr){2-6}\cmidrule(lr){7-11}\cmidrule(l){12-16}
Method & Ret.$\uparrow$ & For. & Test$\uparrow$ & MIA & AG$\downarrow$ & Ret.$\uparrow$ & For. & Test$\uparrow$ & MIA & AG$\downarrow$ & Ret.$\uparrow$ & For. & Test$\uparrow$ & MIA & AG$\downarrow$ \\
\specialrule{0.45pt}{1pt}{2pt}
\addlinespace[1pt]
\multicolumn{16}{c}{\textit{CIFAR-10 / CNN}} \\
\cmidrule(lr){1-16}
\textit{Retrain} & \msd{84.62}{0.18} & \msd{81.72}{0.80} & \msd{81.09}{0.23} & \msd{49.53}{1.50} & 0.00 & \msd{85.32}{0.38} & \msd{81.50}{0.82} & \msd{81.74}{0.24} & \msd{50.42}{0.52} & 0.00 & \msd{84.88}{0.72} & \msd{81.28}{0.82} & \msd{81.54}{0.63} & \msd{49.80}{1.24} & 0.00 \\
\cmidrule(lr){1-16}
Fine-tune & \msd{85.23}{1.10} & \msd{84.14}{1.56} & \msd{82.01}{0.60} & \msd{50.79}{1.00} & 1.30 & \msd{86.06}{1.32} & \msd{83.48}{1.24} & \msd{81.85}{0.92} & \msd{51.73}{1.16} & 1.04 & \msd{87.23}{0.24} & \msd{85.46}{1.36} & \msd{81.98}{0.37} & \msd{48.82}{1.52} & 1.99 \\
FedEraser & \msd{82.00}{2.00} & \msd{82.38}{0.43} & \msd{79.87}{1.47} & \msd{50.48}{1.52} & 1.36 & \msd{82.76}{1.30} & \msd{80.40}{0.48} & \msd{79.91}{1.42} & \msd{50.42}{1.32} & 1.37 & \msd{75.12}{2.42} & \msd{71.59}{0.33} & \msd{72.65}{2.35} & \msd{51.37}{1.36} & 7.48 \\
FedRecovery & \msd{84.84}{0.36} & \msd{85.02}{0.72} & \msd{81.55}{0.27} & \msd{50.49}{0.84} & 1.24 & \msd{84.84}{1.25} & \msd{85.02}{0.42} & \msd{81.56}{1.12} & \msd{50.43}{0.94} & 1.05 & \msd{84.84}{0.36} & \msd{85.02}{0.53} & \msd{81.56}{0.73} & \msd{50.49}{0.35} & 1.12 \\
SalUn & \msd{81.88}{2.62} & \msd{82.60}{1.75} & \msd{79.78}{1.73} & \msd{50.56}{3.13} & 1.49 & \msd{83.35}{1.22} & \msd{82.72}{0.55} & \msd{79.64}{1.03} & \msd{49.16}{4.12} & 1.64 & \msd{86.45}{1.62} & \msd{85.80}{1.35} & \msd{82.88}{1.35} & \msd{48.95}{2.35} & 2.07 \\
MoDe & \msd{71.97}{6.38} & \msd{71.37}{6.11} & \msd{70.86}{6.18} & \msd{49.57}{0.52} & 8.32 & \msd{78.05}{3.71} & \msd{77.31}{2.51} & \msd{75.74}{2.08} & \msd{50.71}{0.92} & 4.44 & \msd{47.55}{3.52} & \msd{51.98}{2.25} & \msd{48.44}{1.25} & \msd{49.03}{0.35} & 25.13 \\
FedOSD & \msd{84.69}{0.49} & \msd{85.24}{1.11} & \msd{81.58}{0.14} & \msd{51.11}{1.86} & 1.42 & \msd{84.71}{0.69} & \msd{84.80}{1.41} & \msd{81.51}{0.44} & \msd{51.19}{1.06} & 1.23 & \msd{84.74}{0.29} & \msd{85.02}{1.26} & \msd{81.44}{0.35} & \msd{51.19}{1.25} & 1.34 \\
NoT & \msd{83.54}{0.37} & \msd{85.40}{1.45} & \msdb{82.15}{0.19} & \msd{51.19}{2.24} & 1.87 & \msd{82.71}{0.92} & \msd{80.28}{1.72} & \msd{79.56}{1.18} & \msd{49.29}{1.04} & 1.79 & \msd{85.18}{0.36} & \msd{84.04}{0.26} & \msd{81.74}{0.24} & \msd{50.01}{1.94} & 0.87 \\
FedCARE & \msd{84.71}{0.35} & \msd{83.48}{0.68} & \msd{81.51}{0.30} & \msd{51.03}{1.04} & \textbf{0.94} & \msd{83.66}{0.65} & \msd{84.48}{1.48} & \msd{80.56}{0.33} & \msd{51.75}{1.08} & 1.79 & \msd{84.79}{0.72} & \msd{83.48}{0.48} & \msd{81.51}{0.54} & \msd{50.57}{0.54} & \textbf{0.77} \\
FedMemPrune & \msd{75.63}{0.74} & \msd{75.77}{0.29} & \msd{73.67}{0.53} & \msd{49.52}{1.31} & 5.59 & \msd{83.42}{1.79} & \msd{81.50}{0.33} & \msd{80.01}{0.43} & \msd{50.23}{0.84} & 0.96 & \msd{72.75}{0.32} & \msd{73.35}{0.59} & \msd{71.21}{0.67} & \msd{49.80}{1.21} & 7.60 \\
\specialrule{0.45pt}{2pt}{1pt}
\rowcolor{black!5} \textbf{SURE (Ours)} & \msdb{85.40}{0.50} & \msd{85.46}{0.90} & \msd{82.05}{0.23} & \msd{51.38}{1.83} & 1.83 & \msdb{86.15}{1.25} & \msd{82.60}{0.80} & \msdb{82.16}{0.33} & \msd{51.38}{1.29} & \textbf{0.83} & \msdb{87.40}{0.66} & \msd{86.56}{0.47} & \msdb{83.18}{0.57} & \msd{50.69}{1.56} & 2.58 \\
\specialrule{0.45pt}{1pt}{2pt}
\addlinespace[1pt]
\multicolumn{16}{c}{\textit{CIFAR-10 / ResNet-18}} \\
\cmidrule(lr){1-16}
\textit{Retrain} & \msd{96.26}{0.38} & \msd{95.37}{0.23} & \msd{87.48}{0.36} & \msd{54.05}{1.47} & 0.00 & \msd{95.92}{0.47} & \msd{89.43}{0.92} & \msd{86.79}{0.43} & \msd{48.32}{0.93} & 0.00 & \msd{96.85}{0.35} & \msd{88.25}{0.93} & \msd{87.60}{0.45} & \msd{49.03}{1.26} & 0.00 \\
\cmidrule(lr){1-16}
Fine-tune & \msd{96.36}{0.14} & \msd{97.04}{0.96} & \msd{87.57}{0.64} & \msd{53.69}{1.27} & 0.56 & \msd{96.24}{1.19} & \msd{93.29}{1.90} & \msdb{87.85}{1.21} & \msd{52.39}{1.43} & 2.33 & \msd{94.97}{0.97} & \msd{93.07}{1.24} & \msd{88.78}{0.26} & \msd{53.70}{1.36} & 3.14 \\
FedEraser & \msd{93.72}{1.45} & \msd{94.71}{1.23} & \msd{85.35}{2.98} & \msd{52.83}{1.25} & 1.64 & \msd{77.75}{1.20} & \msd{76.65}{0.92} & \msd{76.96}{1.52} & \msd{48.31}{1.02} & 10.20 & \msd{58.13}{1.88} & \msd{56.97}{0.36} & \msd{57.83}{1.84} & \msd{50.63}{1.21} & 25.34 \\
FedRecovery & \msdb{96.43}{1.32} & \msd{97.14}{0.80} & \msd{87.81}{1.22} & \msd{54.48}{0.67} & 0.68 & \msd{96.43}{1.06} & \msd{97.14}{1.02} & \msd{87.81}{0.22} & \msd{54.41}{0.94} & 3.83 & \msd{96.21}{0.46} & \msd{96.84}{0.36} & \msd{87.83}{0.24} & \msd{55.10}{0.62} & 3.88 \\
SalUn & \msd{93.41}{1.32} & \msd{92.95}{1.99} & \msd{85.15}{0.56} & \msd{53.38}{2.11} & 2.07 & \msd{91.55}{2.12} & \msd{88.55}{1.28} & \msd{84.68}{1.25} & \msd{50.40}{2.36} & 2.36 & \msd{91.92}{2.26} & \msd{85.44}{0.35} & \msd{84.34}{1.62} & \msd{52.12}{2.24} & 3.52 \\
MoDe & \msd{89.33}{0.77} & \msd{88.11}{2.37} & \msd{82.41}{2.46} & \msd{53.00}{1.13} & 5.08 & \msd{63.88}{4.92} & \msd{67.18}{3.21} & \msd{64.66}{3.32} & \msd{49.59}{0.74} & 19.42 & \msd{46.56}{2.98} & \msd{48.16}{2.36} & \msd{46.78}{1.89} & \msd{49.80}{0.72} & 32.99 \\
FedOSD & \msd{96.36}{0.28} & \msd{96.92}{0.51} & \msd{87.81}{0.44} & \msd{53.85}{0.89} & 0.55 & \msdb{96.56}{1.43} & \msd{96.26}{1.91} & \msd{87.71}{0.25} & \msd{54.37}{0.83} & 3.61 & \msd{96.15}{0.35} & \msd{96.25}{1.23} & \msd{87.62}{0.32} & \msd{55.16}{0.96} & 3.71 \\
NoT & \msd{93.70}{0.49} & \msd{93.39}{1.18} & \msd{85.87}{0.53} & \msd{53.72}{2.37} & 1.62 & \msd{95.80}{0.38} & \msd{93.51}{0.91} & \msd{87.50}{0.63} & \msd{51.41}{2.74} & 2.00 & \msd{92.75}{0.64} & \msd{92.31}{0.25} & \msdb{89.53}{0.16} & \msd{54.69}{1.78} & 3.94 \\
FedCARE & \msd{96.36}{0.72} & \msd{97.14}{1.48} & \msdb{87.83}{1.32} & \msd{54.35}{0.14} & 0.63 & \msd{96.46}{0.94} & \msd{96.92}{0.28} & \msd{87.77}{0.60} & \msd{54.05}{1.72} & 3.69 & \msd{96.19}{0.63} & \msd{96.77}{0.36} & \msd{87.80}{0.23} & \msd{55.04}{1.06} & 3.85 \\
FedMemPrune & \msd{91.79}{0.37} & \msd{92.95}{1.22} & \msd{84.97}{0.55} & \msd{51.68}{1.52} & 2.94 & \msd{92.16}{0.24} & \msd{88.33}{0.93} & \msd{84.80}{0.32} & \msd{50.40}{1.21} & 2.23 & \msd{87.36}{0.46} & \msd{83.85}{0.36} & \msd{81.44}{0.55} & \msd{50.36}{1.24} & 5.35 \\
\specialrule{0.45pt}{2pt}{1pt}
\rowcolor{black!5} \textbf{SURE (Ours)} & \msd{96.09}{0.21} & \msd{96.48}{0.39} & \msd{87.34}{0.42} & \msd{53.69}{1.09} & \textbf{0.45} & \msd{96.17}{0.73} & \msd{92.73}{1.80} & \msd{87.57}{0.63} & \msd{51.72}{1.73} & \textbf{1.93} & \msdb{97.17}{0.72} & \msd{94.13}{0.57} & \msd{88.01}{0.22} & \msd{53.77}{1.52} & \textbf{2.84} \\
\specialrule{0.45pt}{1pt}{2pt}
\addlinespace[1pt]
\multicolumn{16}{c}{\textit{CIFAR-100 / ResNet-18}} \\
\cmidrule(lr){1-16}
\textit{Retrain} & \msd{98.24}{1.17} & \msd{97.98}{0.54} & \msd{57.25}{0.64} & \msd{79.08}{1.15} & 0.00 & \msd{99.05}{1.13} & \msd{54.55}{1.35} & \msd{56.94}{0.24} & \msd{51.40}{0.79} & 0.00 & \msd{98.81}{0.99} & \msd{54.34}{1.38} & \msd{55.65}{0.41} & \msd{49.81}{0.62} & 0.00 \\
\cmidrule(lr){1-16}
Fine-tune & \msd{98.08}{0.52} & \msd{98.79}{0.64} & \msd{57.28}{0.47} & \msd{78.01}{0.56} & 0.52 & \msd{98.83}{0.87} & \msd{84.44}{1.32} & \msd{57.32}{0.22} & \msd{66.19}{0.77} & 11.32 & \msdb{99.30}{1.12} & \msd{82.83}{1.32} & \msdb{57.72}{0.44} & \msd{64.90}{0.46} & 11.54 \\
FedEraser & \msd{96.93}{0.41} & \msd{97.78}{0.68} & \msd{56.05}{0.72} & \msd{78.16}{1.17} & 0.91 & \msd{49.62}{1.02} & \msd{44.65}{0.47} & \msd{45.16}{0.74} & \msd{49.50}{0.66} & 18.25 & \msd{7.80}{0.87} & \msd{9.09}{1.50} & \msd{7.33}{0.36} & \msd{50.29}{0.27} & 46.27 \\
FedRecovery & \msd{98.06}{0.37} & \msd{98.99}{0.92} & \msd{57.34}{0.82} & \msd{79.07}{1.08} & 0.32 & \msd{98.06}{0.76} & \msd{98.99}{1.37} & \msdb{57.33}{0.51} & \msd{79.05}{0.94} & 18.37 & \msd{98.06}{0.85} & \msd{98.99}{1.06} & \msd{57.33}{0.24} & \msd{79.08}{0.61} & 19.09 \\
SalUn & \msd{95.94}{0.78} & \msd{95.15}{1.18} & \msd{55.26}{0.68} & \msd{76.47}{0.29} & 2.43 & \msd{96.01}{0.82} & \msd{62.02}{0.79} & \msd{56.48}{0.75} & \msd{46.54}{0.40} & \textbf{3.96} & \msd{75.96}{1.16} & \msd{81.41}{1.11} & \msd{55.78}{0.65} & \msd{47.38}{0.91} & 13.12 \\
MoDe & \msd{93.51}{0.68} & \msd{94.75}{1.23} & \msd{54.04}{0.85} & \msd{76.00}{0.94} & 3.56 & \msd{27.59}{1.08} & \msd{30.10}{0.81} & \msd{26.72}{0.70} & \msd{50.99}{0.82} & 31.64 & \msd{22.16}{0.82} & \msd{23.84}{0.51} & \msd{21.12}{0.78} & \msd{51.25}{0.73} & 35.78 \\
FedOSD & \msd{98.13}{0.73} & \msd{98.99}{0.62} & \msd{57.17}{0.70} & \msd{79.34}{0.78} & 0.37 & \msd{98.06}{0.52} & \msd{97.37}{0.27} & \msd{57.21}{0.60} & \msd{77.90}{0.45} & 17.65 & \msd{98.13}{1.12} & \msd{97.37}{0.74} & \msd{57.28}{0.49} & \msd{77.71}{1.22} & 18.31 \\
NoT & \msd{95.40}{1.15} & \msd{94.75}{1.01} & \msd{55.45}{0.85} & \msd{76.66}{0.27} & 2.57 & \msd{98.87}{0.71} & \msd{81.82}{0.82} & \msd{56.20}{0.79} & \msd{64.61}{1.03} & 10.35 & \msd{88.56}{0.24} & \msd{79.58}{0.30} & \msd{56.40}{0.41} & \msd{63.51}{0.88} & 12.49 \\
FedCARE & \msd{98.11}{0.34} & \msd{98.99}{0.94} & \msd{57.25}{0.49} & \msd{79.11}{0.29} & 0.29 & \msd{98.13}{0.75} & \msd{98.79}{1.43} & \msd{57.29}{0.43} & \msd{79.13}{0.73} & 18.31 & \msd{98.13}{0.88} & \msd{98.79}{0.85} & \msd{57.28}{0.94} & \msd{79.13}{0.46} & 19.02 \\
FedMemPrune & \msd{94.05}{0.40} & \msd{93.94}{0.67} & \msd{54.03}{0.65} & \msd{76.04}{0.20} & 3.62 & \msd{83.75}{0.36} & \msd{63.23}{1.53} & \msd{52.70}{0.85} & \msd{56.03}{0.47} & 8.21 & \msd{75.05}{1.11} & \msd{51.52}{0.50} & \msd{44.81}{0.80} & \msd{52.90}{1.21} & \textbf{10.13} \\
\specialrule{0.45pt}{2pt}{1pt}
\rowcolor{black!5} \textbf{SURE (Ours)} & \msdb{98.22}{0.59} & \msd{98.79}{0.49} & \msdb{57.48}{0.63} & \msd{79.04}{1.12} & \textbf{0.28} & \msdb{98.96}{0.40} & \msd{87.27}{0.80} & \msd{57.29}{0.89} & \msd{67.46}{0.36} & 12.31 & \msd{98.94}{0.48} & \msd{84.24}{1.51} & \msd{57.25}{0.79} & \msd{66.90}{0.38} & 12.18 \\
\bottomrule
\end{tabular}%
}
\endgroup
\end{table*}

\begin{figure*}[!h]
\centering
\includegraphics[width=0.94\textwidth]{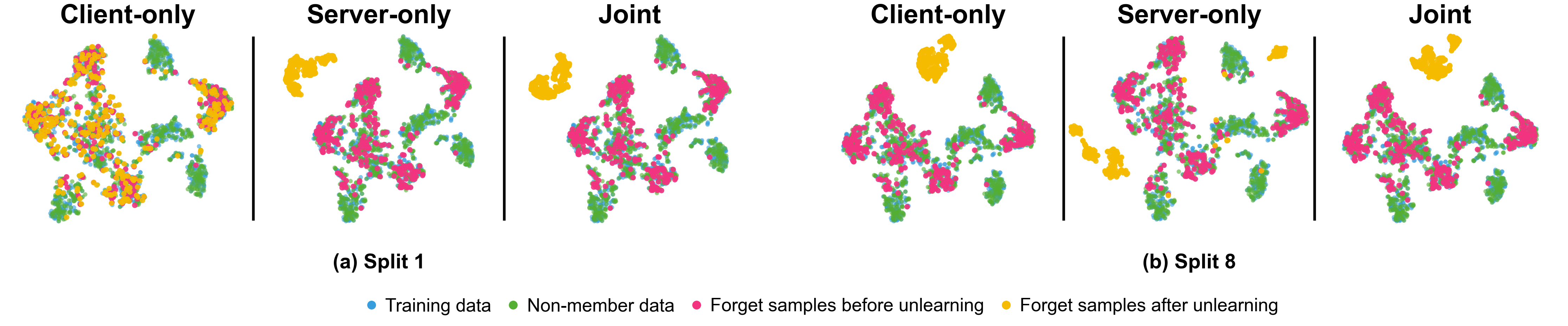}
\setlength{\abovecaptionskip}{1pt}
\caption{t-SNE visualization of client-wise unlearning on CIFAR-10/ResNet-18 at split 1 (left) and split 8 (right). Each row compares Client-only, Server-only, and Joint updates. Blue and green denote retained-training and non-member samples, respectively, while pink and yellow indicate forgotten samples before and after SURE.}
\label{fig:sure-clientwise-tsne}
\end{figure*}

\begin{figure*}[!h]
\centering
\includegraphics[width=0.94\textwidth]{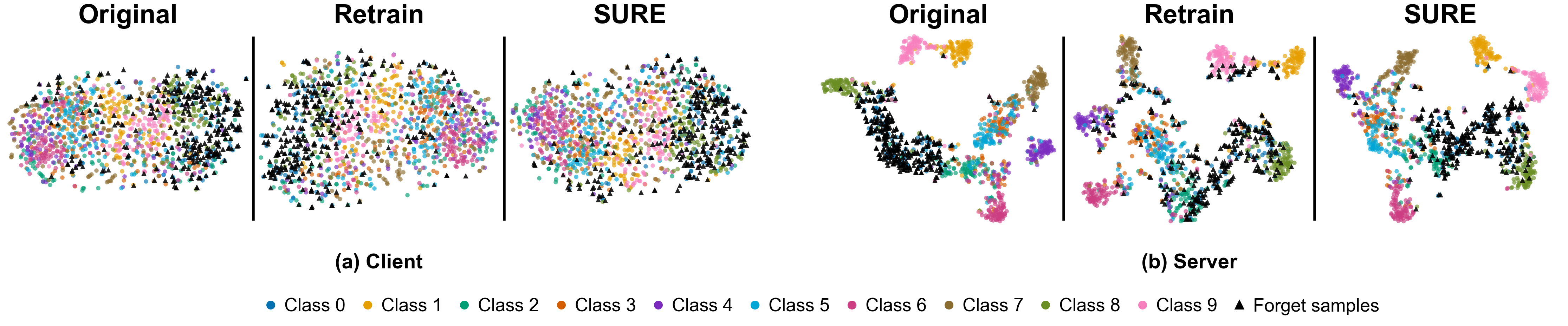}
\setlength{\abovecaptionskip}{1pt}

\caption{t-SNE visualization of class-wise unlearning on CIFAR-10/ResNet-18 under Joint updates. Panels (a) and (b) show client cut-layer and server pre-classifier representations, respectively, each comparing Original, Retrain, and SURE. Colors denote classes, and black triangles mark forgotten class-0 samples.}

\label{fig:sure-cifar-rep}
\end{figure*}

\subsection{Additional Representation Analysis}
\label{app:additional-representations}

\textbf{\textit{Client-Wise Representation Responses.}}
To complement the split-position analysis, we visualize forgotten-sample representations at shallow (split 1) and deep (split 8) cuts on CIFAR-10/ResNet-18. As shown in Fig.~\ref{fig:sure-clientwise-tsne}, forgotten samples remain largely intermixed with other samples under Client-only updates at split 1, whereas Server-only and Joint updates produce more distinct post-unlearning neighborhoods. At split 8, Client-only updates also yield a more concentrated target-sample distribution. These observations illustrate how cut placement and update permissions affect representation responses, although t-SNE separation alone does not establish forgetting effectiveness.

\textbf{\textit{Class-Wise Representation Responses.}}
We further examine client cut-layer and server pre-classifier representations under Joint updates on CIFAR-10/ResNet-18. Fig.~\ref{fig:sure-cifar-rep} shows relatively mixed class distributions at the client partition, whereas the server partition exhibits more structured class neighborhoods. SURE and Retrain display different target-sample arrangements, suggesting that similar prediction-level forgetting behavior does not necessarily imply identical learned representations. These qualitative observations complement the performance results without establishing representation-level equivalence.

\subsection{Multiple Forgetting Clients}
\label{app:additional-multiclient}



We evaluate SURE with multiple client removals on CIFAR-10/ResNet-18 under Dirichlet non-IID partitions ($\alpha=0.1$), where $T$--$F$ denotes $T$ total clients and $F$ forgetting clients. Each setting uses its corresponding mode-matched Retrain reference. As shown in Table~\ref{tab:sure-many}, increasing the forgetting fraction from 10\% (10--1) to 50\% (10--5) reduces Test accuracy and increases Avg. Gap across all update modes. The effect is particularly pronounced under Joint updates, where Test accuracy drops from 77.36\% to 56.42\% and Avg. Gap rises from 1.90 to 4.36 pp. Comparisons among 10--1, 20--2, and 50--5 further reveal mode-dependent behavior at a fixed forgetting ratio. These results demonstrate the challenges of maintaining utility and retraining consistency under larger removal requests.

\begin{table}[t]
\centering
\caption{SURE with multiple forgetting clients on CIFAR-10/ResNet-18 under Dirichlet $\alpha=0.1$. A configuration $T$--$F$ denotes $T$ total clients and $F$ forgetting clients. Avg. Gap (pp) is computed against the corresponding Retrain reference.}
\label{tab:sure-many}
\begingroup
\scriptsize
\setlength{\tabcolsep}{2pt}
\renewcommand{\arraystretch}{0.1}
\begin{tabular}{>{\centering\arraybackslash}p{18mm}lccccc}
\toprule
Mode & Metric & 10--1 & 10--2 & 10--5 & 20--2 & 50--5 \\
\midrule
\multirow{5}{18mm}{\centering\bfseries Client-only}
& Retain (\%) & 81.23$\pm$2.50 & 81.78$\pm$0.51 & 72.37$\pm$0.83 & 72.42$\pm$0.83 & 61.19$\pm$0.76 \\
& Forget (\%) & 67.83$\pm$3.74 & 69.05$\pm$1.22 & 70.98$\pm$1.28 & 74.67$\pm$0.66 & 51.69$\pm$0.36 \\
& Test (\%)   & 76.89$\pm$2.09 & 76.87$\pm$0.77 & 72.21$\pm$0.44 & 71.41$\pm$0.71 & 58.33$\pm$0.59 \\
& MIA (\%)    & 50.87$\pm$2.15 & 51.62$\pm$0.55 & 51.80$\pm$0.47 & 51.19$\pm$0.79 & 50.25$\pm$0.94 \\
& Avg. Gap $\downarrow$ & 0.41 & 0.91 & 1.96 & 0.82 & 0.81 \\
\midrule
\multirow{5}{18mm}{\centering\bfseries Server-only}
& Retain (\%) & 81.16$\pm$1.99 & 82.36$\pm$0.52 & 75.18$\pm$0.37 & 75.22$\pm$0.71 & 57.63$\pm$0.35 \\
& Forget (\%) & 67.43$\pm$4.39 & 70.38$\pm$1.04 & 51.75$\pm$0.36 & 76.17$\pm$0.62 & 52.33$\pm$0.71 \\
& Test (\%)   & 76.54$\pm$4.38 & 75.99$\pm$0.30 & 57.70$\pm$0.23 & 72.56$\pm$0.62 & 55.15$\pm$0.25 \\
& MIA (\%)    & 51.22$\pm$1.33 & 50.98$\pm$0.82 & 51.06$\pm$0.16 & 50.64$\pm$0.32 & 49.40$\pm$0.87 \\
& Avg. Gap $\downarrow$ & 1.63 & 2.12 & 2.93 & 1.16 & 4.34 \\
\midrule
\multirow{5}{18mm}{\centering\bfseries Joint}
& Retain (\%) & 81.76$\pm$2.17 & 81.65$\pm$0.48 & 67.88$\pm$0.99 & 74.94$\pm$0.86 & 63.50$\pm$0.52 \\
& Forget (\%) & 65.34$\pm$4.93 & 72.19$\pm$1.27 & 52.44$\pm$0.96 & 75.17$\pm$1.30 & 51.85$\pm$0.62 \\
& Test (\%)   & 77.36$\pm$1.22 & 74.97$\pm$0.64 & 56.42$\pm$0.78 & 72.26$\pm$0.51 & 60.39$\pm$0.54 \\
& MIA (\%)    & 48.61$\pm$1.94 & 50.55$\pm$0.25 & 51.15$\pm$0.26 & 50.36$\pm$0.23 & 48.83$\pm$0.37 \\
& Avg. Gap $\downarrow$ & 1.90 & 2.00 & 4.36 & 3.63 & 3.78 \\
\bottomrule
\end{tabular}%
\endgroup
\end{table}

\begin{figure*}[]
\centering
\includegraphics[width=\textwidth]{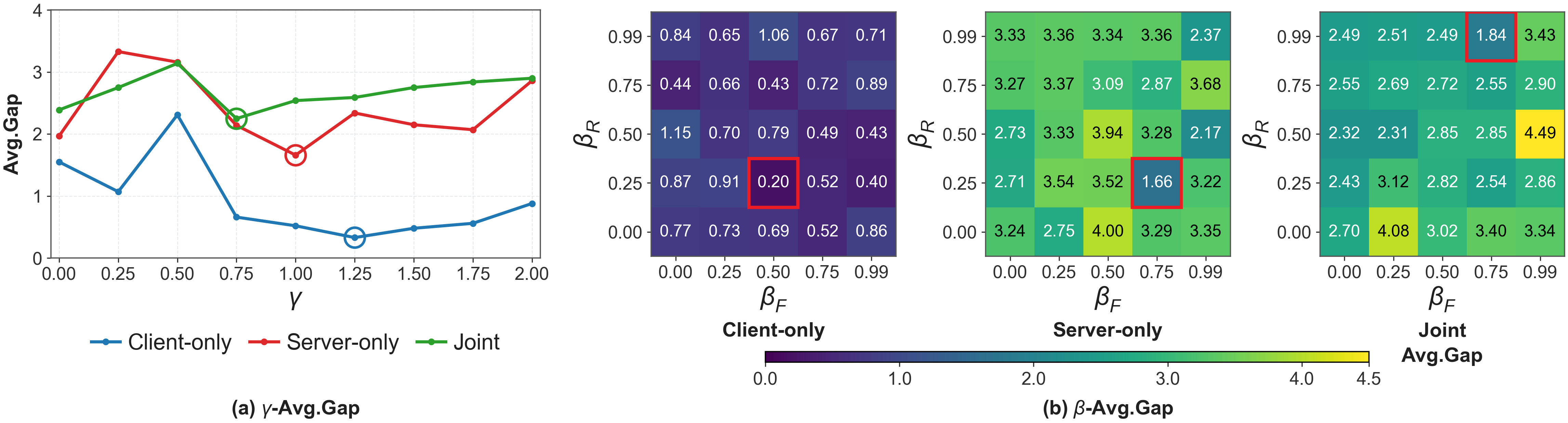}
\setlength{\abovecaptionskip}{-5pt}
\caption{Sensitivity of SURE to perturbation protection and recovery coefficients on CIFAR-10/ResNet-18 under Dirichlet $\alpha=0.1$. (a) Avg. Gap as the perturbation protection coefficient $\gamma$ varies under Client-only, Server-only, and Joint updates. (b) Avg. Gap over recovery coefficients $(\beta_R,\beta_F)$ for the three update modes. Lower is better. Circles and red boxes mark the best setting within each mode.}
\label{fig:sure-sensitivity}
\end{figure*}

\subsection{Hyperparameter Sensitivity}
\label{app:additional-sensitivity}

We evaluate the sensitivity of SURE to the perturbation protection coefficient $\gamma$ and the recovery coefficients $(\beta_R,\beta_F)$ on CIFAR-10/ResNet-18. Fig.~\ref{fig:sure-sensitivity}(a) shows that the best $\gamma$ differs across update modes: the minimum Avg. Gap is obtained at $\gamma=1.25$ for Client-only, $\gamma=1.0$ for Server-only, and $\gamma=0.75$ for Joint updates. Fig.~\ref{fig:sure-sensitivity}(b) shows a similar mode dependence for the recovery coefficients. The best setting is $(\beta_R,\beta_F)=(0.25,0.50)$ for Client-only, $(0.25,0.75)$ for Server-only, and $(0.99,0.75)$ for Joint updates. Since no single configuration is uniformly optimal across all modes, we use the shared default $(\gamma,\beta_R,\beta_F)=(1,0.25,0.75)$ in the main experiments as a practical compromise.

\section{Proofs of Theoretical Results}
\label{app:theory}

\subsection{Proof of Theorem~\ref{thm:influence_propagation}}
\label{app:proof-influence}

\begin{proof}
Add and subtract $\mathcal{G}_t(z_R^t)$ in the trajectory difference. By the fundamental theorem of calculus and the definitions of $A_t$ and $b_t$,
\begin{equation}
\delta_{t+1}
=\mathcal{G}_t(z^t)-\mathcal{G}_t(z_R^t)+b_t
=A_t\delta_t+b_t.
\end{equation}
Iterating from the shared initialization $\delta_0=0$ gives Eq.~\eqref{eq:sfu_influence_decomposition}. The triangle inequality and submultiplicativity yield
\begin{equation}
\|\delta_T\|_2\le
\sum_{t=0}^{T-1}
\left(\prod_{r=t+1}^{T-1}\|A_r\|_2\right)
\|b_t\|_2,
\end{equation}
which implies Eq.~\eqref{eq:sfu_influence_bound} using $\|A_r\|_2\le L_r$.
\end{proof}

\subsection{Proof of Theorem~\ref{thm:feasibility_bound}}
\label{app:proof-feasibility}

\begin{proof}
For any $\Delta\theta\in\Theta_{\mathcal{M}}$ with $\|\Delta\theta\|_2\le\rho$, write
\begin{equation}
F(\theta^*+\Delta\theta)
=F(\theta^*)+DF(\theta^*)[\Delta\theta]+R_\Delta.
\end{equation}
The $L$-Lipschitz derivative assumption gives
\begin{equation}
\|R_\Delta\|_{\mathcal{H}}
\le\int_0^1 L\alpha\|\Delta\theta\|_2^2\,d\alpha
\le\frac{L\rho^2}{2}.
\end{equation}
Since $DF(\theta^*)[\Delta\theta]\in\mathcal{V}_{\mathcal{M}}$, orthogonal projection yields
\begin{equation}
\begin{aligned}
&\|F(\theta^*+\Delta\theta)-F(\theta^R)\|_{\mathcal H}\\
&\quad\ge
\|DF(\theta^*)[\Delta\theta]-r\|_{\mathcal H}
-\|R_\Delta\|_{\mathcal H}\\
&\quad\ge d_{\mathcal M}-\frac{L\rho^2}{2}.
\end{aligned}
\end{equation}
Taking the nonnegative part and infimum over admissible updates proves Eq.~\eqref{eq:sfu_feasibility_bound}.
\end{proof}

\subsection{Proof of Corollary~\ref{cor:mode_hierarchy}}
\label{app:proof-hierarchy}

\begin{proof}
The inclusions
$\Theta_{\{C\}},\Theta_{\{S\}}\subseteq\Theta_{\{C,S\}}$
imply the same inclusions for their images under $DF(\theta^*)$. Thus, the projection residual cannot increase under Joint updates. The corresponding nonlinear feasible sets are also nested for a common $\rho$; taking infima establishes the feasibility-error hierarchy in Eq.~\eqref{eq:sfu_feasibility_hierarchy}.
\end{proof}

\subsection{Proof of Proposition~\ref{prop:excess_risk}}
\label{app:proof-risk}

\begin{proof}
For Client-only updates, set $u_C=\theta_C-\theta_C^R$ and $e_S=\theta_S^*-\theta_S^R$. The quadratic excess risk is
\begin{equation}
\begin{aligned}
Q_R-\mathcal{L}_R(\theta^R)
&=\tfrac12 u_C^\top H_{CC}u_C
+u_C^\top H_{CS}e_S\\
&\quad+\tfrac12 e_S^\top H_{SS}e_S.
\end{aligned}
\end{equation}
Since $H_{CC}\succ0$, its unique minimizer satisfies
$u_C^{\mathrm{opt}}=-H_{CC}^{-1}H_{CS}e_S$.
Substitution gives Eq.~\eqref{eq:sfu_client_risk_gap}. Exchanging $C$ and $S$ yields Eq.~\eqref{eq:sfu_server_risk_gap}. Both Schur complements are positive definite because $H_R\succ0$, proving nonnegativity and strict positivity for nonzero frozen-partition displacement.
\end{proof}

\section{Additional Experimental Details}
\label{app:experimental-details}

\subsection{Experimental Settings}
\label{app:experimental-settings}


\textbf{\textit{Data and Model Splits.}}
Extreme non-IID partitions assign one exclusive class per client on MNIST/CIFAR-10 and ten classes per client on CIFAR-100. Class-wise, ablation, sensitivity, and multi-client experiments use Dirichlet $\alpha=0.1$. Client partitions contain the first convolution-pooling block of LeNet, the first two convolutional blocks of CNN, or the ResNet-18 stem and first residual block (split 2). Nine ResNet-18 cut positions are examined separately.

\textbf{\textit{Training Configurations.}}
Original and Retrain share identical round budgets: 50 for MNIST, 100/500 for CIFAR-10/CNN, and 200/1000 for ResNet-18 under IID/extreme non-IID partitions, respectively. Class-wise experiments use 50 (MNIST) or 200 (CIFAR-10) rounds. Each client performs one local epoch using SGD (momentum $0.9$), with batch sizes 32 (MNIST) or 128 (CIFAR). SURE uses batch size 32, at most 40 recovery rounds, and $(\gamma,\beta_R,\beta_F)=(1,0.25,0.75)$. MNIST uses learning rate $0.005$ and weight decay $10^{-4}$. IID CIFAR uses $0.05$ and $5\times10^{-4}$, with cosine decay to $0.001$ after 80\% of training. Extreme non-IID CIFAR uses $0.005$ and $10^{-4}$, with exponential decay ($0.998$ for CNN; $0.999$ for ResNet-18) to a minimum of $0.0002$. MNIST uses normalization only; CIFAR additionally applies random cropping and horizontal flipping.



\textbf{\textit{Retraining and Checkpoint Selection.}}
Let $\theta^{\mathrm{init}}$ and $\theta^*$ denote the original initialization and trained checkpoint. Client-only, Server-only, and Joint Retrain initialize from $(\theta_C^{\mathrm{init}},\theta_S^*)$, $(\theta_C^*,\theta_S^{\mathrm{init}})$, and $\theta^{\mathrm{init}}$, respectively, using retained data while preserving frozen parameters and buffers. Main experiments use seeds 41--43; sensitivity and representation analyses use seed 42. Original and Retrain select checkpoints by validation accuracy, breaking ties by earlier rounds. Unlearning selects the highest-validation-accuracy checkpoint satisfying the forgetting constraint; otherwise, the checkpoint with the lowest target validation accuracy is reported as failing the constraint. Test data are excluded from selection.

\subsection{Baseline Adaptation to SFL}
\label{app:sfl-baseline-adaptation}



All baselines share SURE's data partitions, model architectures, and unlearning requests, with their operations adapted to split computation and dual-side aggregation. Parameter modifications are restricted to the permitted update modes, while frozen partitions retain their parameters and persistent states. Recovery-based methods follow the original SFL aggregation protocol, and history-dependent methods use records from the training runs. FedEraser and FedRecovery are evaluated only for client-wise unlearning.

\subsection{Membership Inference Attack}
\label{app:mia}




We use a gradient-boosting attacker with label-aware prediction features (probabilities, entropy, and margins). Forgotten training samples serve as members and class-matched unseen samples as non-members. Balanced attack data are split into disjoint training, calibration, and evaluation sets. The attacker is fitted and calibrated on Original, then held fixed across Retrain and unlearning methods. ROC-AUC measures membership distinguishability, with 50\% indicating chance-level discrimination rather than certified forgetting.

\subsection{Metric and Cost Accounting}
\label{app:metric-cost-protocol}

For class-wise unlearning, Forget measures class-0 test accuracy, while Test includes all classes.
Communication includes model transfers, cut-layer activations, and gradients, while computation is estimated in FLOPs. Client-wise and ablation costs reflect executed procedures, whereas class-wise costs use a fixed 40-round budget. 

\end{document}